\documentclass[twoside,11pt,capitalize,nameinlink]{article}
\usepackage[preprint]{jmlr2e}

\usepackage[linesnumbered,ruled]{algorithm2e}

\usepackage{mathtools}
\usepackage{bbm}
\usepackage[shortlabels]{enumitem}

\usepackage{algpseudocodex}
\usepackage{tipa}

\algrenewcommand\algorithmicrequire{\textbf{Input:}}
\algrenewcommand\algorithmicensure{\textbf{Output:}}

\renewcommand{\epsilon}{\varepsilon}

\usepackage{thmtools,thm-restate}

\usepackage{parskip}

\usepackage{pgfplots}
\pgfplotsset{compat=1.15}
\usepackage{mathrsfs}
\usetikzlibrary{arrows}
\usetikzlibrary{shapes}

\usepackage{xfrac}

\DeclareMathOperator{\C}{\mathfrak C}
\DeclareMathOperator{\cC}{\mathcal C}
\DeclareMathOperator{\cL}{\mathcal L}
\DeclareMathOperator{\cU}{\mathcal U}
\DeclareMathOperator{\cT}{\mathcal T}
\DeclareMathOperator{\conv}{\mathfrak{c}} %
\newcommand{\cm}{\ensuremath{\mathfrak{m}}\xspace}
\newcommand{\cM}{\ensuremath{\mathfrak{M}}\xspace}
\DeclareMathOperator{\Imp}{\iota}

\DeclareMathOperator{\vcdim}{VCdim}
\DeclareMathOperator{\TD}{TD}
\DeclareMathOperator{\TS}{TS}
\DeclareMathOperator{\cTS}{\mathcal T \mathcal S}
\DeclareMathOperator{\RTD}{RTD}
\newcommand{\Sam}{\mathrm{Samp}}

\DeclareMathOperator{\ord}{\mathrm{ord}}

\DeclareMathOperator{\Hm}{{\mathcal H}_m}
\DeclareMathOperator{\Hml}{{\mathcal H}_m^\ell}
\DeclareMathOperator{\wHm}{\widehat{{\mathcal H}}_m}
\DeclareMathOperator{\Hmm}{{\mathcal H}^+_m}
\DeclareMathOperator{\hullm}{{hull}_m}
\DeclareMathOperator{\hull}{hull}

\DeclareMathOperator{\twin}{twin}
\DeclareMathOperator{\CC}{\mathcal{C}}
\newcommand{\covectors}{\ensuremath{\mathcal{L}}}
\newcommand{\cH}{\ensuremath{H^{\scriptstyle c}}}

\SetKwInput{KwInput}{Input}
\SetKwInput{KwOutput}{Output}
\SetKw{Continue}{continue}
\DeclareMathOperator{\oGamma}{{\Gamma}}

\title{Efficient Algorithms for Learning and Compressing Monophonic Halfspaces in Graphs}

\author{%
\name{Marco Bressan} \email{marco.bressan@unimi.it}\\
\addr Università degli Studi di Milano, Milan, Italy%
\AND
\name{Victor Chepoi} \email{victor.chepoi@lis-lab.fr}\\
\addr Universit\'e d'Aix-Marseille and CNRS, Marseille, France%
\AND
\name{Emmanuel Esposito} \email{emmanuel@emmanuelesposito.it}\\
\addr Università degli Studi di Milano, Milan, Italy%
\AND
\name{Maximilian Thiessen} \email{maximilian.thiessen@tuwien.ac.at}\\
\addr TU Wien, Vienna, Austria%
}
\editor{}

\crefname{algocf}{alg.}{algs.}
\Crefname{algocf}{Algorithm}{Algorithms}

\newcommand{\eps}{\varepsilon}
\newcommand{\scC}{\mathcal{C}} %
\newcommand{\scH}{\mathcal{H}} %
\newcommand{\scO}{\mathcal{O}}

\newcommand{\MonInt}{I}

\newcommand{\algoname}[1]{\ensuremath{\textsc{#1}}\xspace}

\newcommand{\HalfspaceSeparation}{\algoname{HalfspaceSep}}

\newcommand{\ShadowClosure}{\algoname{ShadowClosure}}

\newcommand{\R}{\mathbb{R}}

\DeclareMathOperator{\VC}{VCdim}

\DeclareMathOperator{\diam}{diam}
\DeclareMathOperator{\poly}{poly}
\DeclareMathOperator{\qc}{qc} %
\DeclareMathOperator*{\argmin}{arg\,min}

\usepackage[framemethod=TikZ]{mdframed}
\newcounter{emmanuelbox}

\newcommand{\NO}{\ensuremath{\mathrm{NO}}}

\newcommand{\ignore}[1]{}

\usepackage{lastpage}
\jmlrheading{}{1-\pageref{LastPage}}{}{}{}{}{Marco Bressan, Victor Chepoi, Emmanuel Esposito, and Maximilian Thiessen}

\ShortHeadings{Learning and compressing monophonic halfspaces in graphs}{Bressan Chepoi Esposito Thiessen}
\firstpageno{1}

\begin{document}

\maketitle

\begin{abstract}%
Abstract notions of convexity over the vertices of a graph, and corresponding notions of halfspaces, have recently gained attention from the machine learning community.
In this work we study \emph{monophonic halfspaces}, a notion of graph halfspaces defined through closure under induced paths.
Our main result is a $2$-satisfiability based decomposition theorem, which allows one to represent monophonic halfspaces as a disjoint union of certain vertex subsets.
Using this decomposition, we achieve efficient and (nearly) optimal algorithms for various learning problems, such as teaching, active, and online learning. 
Most notably, we obtain a polynomial-time algorithm for empirical risk minimization. %
Independently of the decomposition theorem, we obtain an efficient, stable, and proper sample compression scheme. This makes monophonic halfspaces %
efficiently learnable with proper learners and linear error rate $\sfrac{1}{\eps}$ in the realizable PAC setting. %
Our results answer open questions from the literature, and show a stark contrast with geodesic halfspaces, for which most of the said learning problems are NP-hard.
\end{abstract}

\section{Introduction}
We study binary classification of the vertices of a graph. With the advent of social networks and graph-based machine learning, this problem has received considerable attention in supervised learning~\citep{hanneke2006analysis,pelckmans2007margin}, active learning~\citep{afshani2007complexity,guillory2009label,cesa2010active,dasarathy2015s2}, and online learning \citep{herbster2005online,cesa2013random,herbster2015online}.
The concept class is often designed to reflect the homophily principle, that is, the tendency of adjacent vertices to belong to the same class.
In this work we take a different perspective, and assume that concepts are realized by \emph{monophonic halfspaces}, an abstract notion of halfspaces related to linear separability and convexity in Euclidean spaces.
In this way we hope to exploit the well-established machinery behind convexity, which is often at the heart of machine learning models (think of intervals, halfspaces, or polytopes).

The interest in graph halfspaces dates back to the '80s, but has been recently revived~\citep{duchet1983ensemble,chepoi1986some,chepoi1994separation,chepoi2024separation,farber1986convexity,pelayo2013geodesic,thiessen2021active,bressan2021exact,chalopin2022unlabeled,seiffarth2023maximal, chepoi2024separation, sokolov2025self}.
Notwithstanding, little is known about efficient algorithm for learning graph halfspaces.
On the one hand, most works focus on abstract properties of graph halfspaces, such as the relation between their invariants (say, the VC-dimension) and the properties of the underlying graph (say, the clique number).
On the other hand, for some standard graph halfspaces (e.g., geodesic halfspaces, which are closed under shortest paths) even deciding if there exists a hypothesis consistent with a given labeled sample is NP-hard \citep{seiffarth2023maximal}.
Therefore, it is not even clear whether nontrivial graph halfspaces exist that can be learned efficiently.
One of the contributions of this work is to show this is the case.

In this work we study \emph{monophonic halfspaces}, a notion of graph halfspaces defined through induced paths \citep{farber1986convexity,duchet1988convex,bandelt1989graphs}.
Let us introduce some notation. Let $G=(V,E)$ be a graph. Given two vertices $x,y \in V$, the \emph{monophonic interval} $I_m(x,y)$ between $x$ and $y$ is the set of all vertices that lie on some induced $(x,y)$-path.
One can see $I_m(x,y)$ as the graph equivalent of the segment connecting two points in $\R^d$.
A set $C \subseteq V$ is \emph{monophonically convex} (m-convex) if $I_m(x,y)\subseteq C$ for all $x,y\in C$. A set $H \subseteq V$ is a \emph{monophonic halfspace} (m-halfspace) if both $H$ and $V \setminus H$ are m-convex. For instance, if $G$ is a tree, then the connected components left by deleting an edge are m-halfspaces; if $G$ is a clique, then any subset is a m-halfspace. \Cref{fig:halfspace} gives another example. In real-world networks, communities and clusters often tend to be geodesically convex, e.g., in gene similarity networks \citep{zhou2002transitive}, protein-protein interaction networks \citep{li2013identification}, community detection benchmark datasets \citep{thiessen2021active}, and collaboration networks \citep{vsubelj2019convexity}. In the latter and many other cases, the set of  monophonic and geodesic convexity actually coincides \citep{malvestuto2012characteristic}.
Monophonic halfspaces are among the most studied graph halfspaces \citep{bandelt1989graphs,changat2005convexities,dourado2010complexity}, second only to geodesic halfspaces. %
Our work provides several concrete results about the structure and learnability of  m-halfspaces.

\subsection{Contributions}
For a graph $G=(V,E)$, let $n=|V|$ and $m=|E|$, let $\omega(G)$ be the clique number,  $\Hm(G)$ be the set of m-halfspaces, and $d$ be the VC-dimension of $\Hm(G)$. For an edge $ab$, we denote by $\Hm(ab)$ the set of m-halfspaces containing $a$ and avoiding $b$. 
\setlist[description]{font=\normalfont\bfseries}
\begin{description}[leftmargin=0in,itemsep=0pt,topsep=0pt]

\item[A polynomial-time algorithm for halfspace separation] (\Cref{sec:polycheck}). 
 Our first contribution is a polynomial-time algorithm for finding a halfspace $H \in \Hm(G)$ that separates two given subsets $A,B \subseteq V$ if one exists, see \Cref{thm:monophonic-halfspace-separation}.
 As a consequence, in the realizable case we obtain a polynomial-time PAC learning algorithm with sample complexity $\scO\left(\frac{d\log(1/\eps)+\log(1/\delta)}{\eps}\right)$, see \Cref{thm:polytime_ralizable_PAC}.
 Our result also implies separation with geodesic convexity, where the halfspace separation problem is NP-complete~\citep{seiffarth2023maximal}, and answers the open question of~\cite{gonzalez2020covering} about the complexity of finding a $k$-partition of $V(G)$ for the case $k=2$.
 Our algorithm relies on a careful reduction to 2-SAT that may be of independent interest.

\item[A decomposition theorem for monophonic halfspaces] (\Cref{sec:decomposition-VC-dim}). 
This is the most technical of our contributions.
Through a careful analysis of the 2-SAT formula mentioned above, we show that any m-halfspace $H\in \Hm(ab)$ can be written as a disjoint union of certain subsets, called \textit{shadows} and \textit{cells}.
This yields what we call the \emph{shadow-cell decomposition} of monophonic halfspaces, see \Cref{thm:unifiedDecomposition}. Shadow-cell decompositions can be computed efficiently, and support useful algorithmic operations such as finding the halfspace that better agrees with some labeled sample (that is, empirical risk minimization).
Indeed, shadow-cell decompositions yield efficient algorithms for several learning tasks, see below.
They also provide a tight characterization of the VC dimension of $\Hm(G)$ up to an additive constant of $4$, see \Cref{thm:connected-components-vc-dim} in \Cref{sec:VC-dim}.

\item[An efficient sample compression scheme] (\Cref{sec:LSCS}). We give a proper labeled sample compression scheme (LSCS) of size $4\omega(G)$ that runs in polynomial time, see \Cref{t:m-halfspaceCompression}. 
Moreover our scheme is \emph{stable}; by \citet{BousquetHMZ20} this implies a polynomial-time realizable PAC learner with sample complexity $\scO\left(\frac{\omega(G)+\log(1/\delta)}{\eps}\right)$, without the extra $\log(1/\eps)$ factor as above. When $d=\Theta(\omega(G))$, this makes m-halfspaces one of the few nontrivial classes where the optimal sample complexity is known to be achievable in polynomial time.

\item[Efficient learning algorithms](\Cref{sec:supervised,sec:active,sec:online,sec:teaching}). 
By exploiting our shadow-cell decomposition, we provide polynomial algorithms for learning monophonic halfspaces in several standard settings.
For supervised learning (\Cref{sec:supervised}), we give an empirical risk minimization (ERM) algorithm that achieves the optimal sample complexity, see \Cref{thm:erm-poly}.
For active learning (\Cref{sec:active}), we give an algorithm that runs in polynomial time and uses $\scO\left(\hullm(G)+\log\diam(G)+d\right)$ queries, see \Cref{thm:active_upper_bound}. Here $\diam(G)$ is the diameter and $\hullm(G)$ the ``m-hull number'' of $G$.
For online learning (\Cref{sec:online}), we obtain an algorithm that makes $\scO(d\log n)$ mistakes by using \textsc{Winnow}, as well as an algorithm that makes $\scO\Bigl(d + \log n\Bigr)$ mistakes (but runs in $2^d\poly(n)$ time) by using \textsc{Halving}; see \Cref{thm:polyAndFPTonline}.
For teaching (\Cref{sec:teaching}), we show that the teaching dimension of $\Hm(G)\setminus \{ \varnothing, V\}$ and the recursive teaching dimension of $\Hm(G)$ are both at most $2d+2$; see \Cref{thm:teaching}. Teaching sets of size $2d+2$ can be efficiently computed by relying on the shadow-cell decomposition.

\end{description}

\subsection{Discussion and related work}\label{sec:related}
We substantially extend the algorithmic results of \citet{BrEsTh24} and the results on halfspace separation of \citet{chepoi2024separation}.
Our results on supervised, active, and online learning of m-halfspaces significantly improve the results of \citet{BrEsTh24}, since we replace the clique number $\omega(G)$ used by \citet{BrEsTh24} with the VC-dimension $d\le \omega(G)$; it is easy to see that $\omega(G)$ can be unbounded as a function of $d$, and in fact one can have $\omega(G) \simeq d \cdot |V(G)|$.
A polynomial-time algorithm for m-halfspace separation was independently obtained by \citet{BrEsTh24},  \citet{ElaroussiNV2024half}, and \citet{chepoi2024separation}; we use the (slightly corrected) third solution.

Our contributions are among the few \emph{constructive} results on the efficiency of learning halfspaces and/or convex sets in graphs, and belong to a line of research on graph hypothesis spaces and their learning properties~\citep{chepoi2007covering,chepoi2021labeled,ChChMcRaVa,ChalopinCIR24,CoudertCDV24,Si25}. Note that almost all of our proposed algorithms with near-optimal guarantees have polynomial runtime on all graphs. %
By contrast previous results only achieve polynomial runtime or tight guarantees by assuming bounded treewidth \citep{thiessen2021active} and/or maximum \emph{clique-minor} size \citep{duchet1983ensemble,chepoi2007covering,ChChMcRaVa,le2023vc}. %

Empirical risk minimization can be reduced to listing  $\Hm(G)$.
If $C$ is m-convex, then the vertices of $V\setminus C$ that are adjacent to $C$ form a clique. Thus one can list $\Hm(G)$ by listing all pairs of cliques of $G$ and checking if the edges between them form a cut of $G$, for a running time of $n^{2\omega(G)}\poly(n)$. A better bound can be achieved using a polynomial-time consistency checker: in that case, by a folklore algorithm one can list $\Hm(G)$ in time $|\Hm(G)| \poly(n)$. Our work gives both a polynomial-time consistency checker and a tight bound on $|\Hm(G)|$; neither one was known before. In particular, bounds on $|\Hm(G)|$ given by standard VC-dimension arguments suffer an exponential dependence on the cutsize (i.e., the number of edges) of the halfspace~\citep{kleinberg2004detecting}. In our case the cutsize can be as large as $\Theta(\omega(G)^2)$, which yields $|\Hm(G)| \le n^{\scO(\omega(G)^2)}$.
This is significantly beaten by our %
bound $|\Hm(G)| \le m2^{d+1}+2$. %
\citet{glantz2017finding} give polynomial time algorithms for enumerating geodesic halfspaces of bipartite and planar graphs, but do not have results for general graphs. By contrast, we can enumerate all m-halfspaces in time $2^d\poly(n)$.

For active learning, \citet{thiessen2021active} give lower and upper bounds, but for \emph{geodesic} halfspaces. Their algorithm requires computing a minimum geodesic hull set, which is APX-hard with no known approximation algorithm. %
in contrast, our algorithm runs in polynomial time. \citet{bressan2021exact} also studied active learning on graphs under a geodesic convexity assumption. They achieved polynomial time however with additional assumptions on the convex sets, such as margin.
For online learning, \citet{thiessen2022online} give again results for geodesic halfspaces. Their algorithms, however, are computationally inefficient and/or loose in terms of mistake bounds. We instead represent monophonic halfspaces as a union of a small number of shadows and cells, which allows us to achieve a near-optimal mistake bound efficiently relying on \textsc{Winnow} \citep{littlestone1988learning}.

Labeled sample compression schemes (LSCS) have been introduced
by \citet{FlWa}, who asked whether
any set family of VC-dimension $d$ has an LSCS %
of size~$O(d)$. This remains one of the oldest open problems in machine learning. The conjecture was confirmed for various hypothesis spaces \citep{Ben-DavidL98,chalopin2022unlabeled, ChChMcRaVa, chepoi2021labeled, FlWa, KuzminW07, MoranW16, RubinsteinR12, RubinsteinR22}.
\citet{MoranY16} proved that any hypothesis space $\CC$ of VC-dimension $d$ admits an LSCS of size $O(2^d)$.

\section{Preliminaries}\label{sec:prelim}
\paragraph{Graphs.}
All graphs $G=(V,E)$ in this paper are
simple, undirected, and connected and let $n=|V|$ and $m=|E|$.
We write $u \sim v$ if $u, v \in V$ are adjacent and $u\nsim v$ if $u,v$ are not adjacent.
Let $N(v) = \{u \in V \,:\, uv \in E\}$ and $N[v]=N(v)\cup \{ v\}$ are the \emph{open} and the \emph{closed neighborhoods} of $v\in V$.  The subgraph of $G=(V,E)$  \emph{induced by} $A\subseteq V$
is the graph $G[A]=(A,E')$ such that $uv\in E'$ if and only if $uv\in E$. A \emph{clique}  is a complete subgraph of $G$.
For an edge $uv$, denote by $\Delta_{uv}$ the set $N[u]\cap N[v]$ and the subgraph induced by this set.
Let $\omega(G)$ denote the clique number of $G$.  %
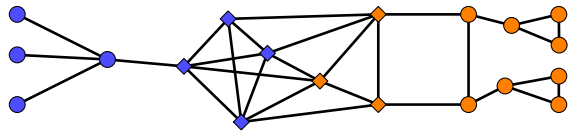
\begin{figure}[tb]
    \centering
    \colorlet{ffttww}{orange}
\definecolor{ududff}{rgb}{0.30196078431372547,0.30196078431372547,1}
\begin{tikzpicture}[line cap=round,line join=round,>=triangle 45,x=1cm,y=1cm,scale=.6]
\draw [line width=1pt] (-2,0)-- (0,-1);
\draw [line width=1pt] (-2,-0.89835)-- (0,-1);
\draw [line width=1pt] (-2,-2)-- (0,-1);
\draw [line width=1pt] (0,-1)-- (1.6941,-1.15245);
\draw [line width=1pt] (2.6742,-0.09975)-- (3.5454,-0.86205);
\draw [line width=1pt] (3.5454,-0.86205)-- (1.6941,-1.15245);
\draw [line width=1pt] (2.6742,-0.09975)-- (2.9646,-2.38665);
\draw [line width=1pt] (2.9646,-2.38665)-- (1.6941,-1.15245);
\draw [line width=1pt] (2.6742,-0.09975)-- (1.6941,-1.15245);
\draw [line width=1pt] (2.9646,-2.38665)-- (3.5454,-0.86205);
\draw [line width=1pt] (3.5454,-0.86205)-- (6,0);
\draw [line width=1pt] (6,0)-- (4.707,-1.47915);
\draw [line width=1pt] (4.707,-1.47915)-- (6,-2);
\draw [line width=1pt] (6,0)-- (6,-2);
\draw [line width=1pt] (6,-2)-- (8,-2);
\draw [line width=1pt] (6,0)-- (8,0);
\draw [line width=1pt] (8.9541,-0.24495)-- (8,0);
\draw [line width=1pt] (10,0)-- (8.9541,-0.24495);
\draw [line width=1pt] (10.0068,-0.68055)-- (8.9541,-0.24495);
\draw [line width=1pt] (8,0)-- (8,-2);
\draw [line width=1pt] (10,0)-- (10.0068,-0.68055);
\draw [line width=1pt] (8,-2)-- (8.8089,-1.58805);
\draw [line width=1pt] (8.8089,-1.58805)-- (10,-1.37025);
\draw [line width=1pt] (8.8089,-1.58805)-- (10,-2);
\draw [line width=1pt] (10,-1.37025)-- (10,-2);
\draw [line width=1pt] (3.5454,-0.86205)-- (4.707,-1.47915);
\draw [line width=1pt] (2.9646,-2.38665)-- (4.707,-1.47915);
\draw [line width=1pt] (2.6742,-0.09975)-- (6,0);
\draw [line width=1pt] (2.9646,-2.38665)-- (6,-2);
\draw [line width=1pt] (1.6941,-1.15245)-- (4.707,-1.47915);
\begin{scriptsize}
\draw [fill=ududff] (3.5454,-0.86205) ++(-5pt,0 pt) -- ++(5pt,5pt)--++(5pt,-5pt)--++(-5pt,-5pt)--++(-5pt,5pt);
\draw [fill=ududff] (2.6742,-0.09975) ++(-5pt,0 pt) -- ++(5pt,5pt)--++(5pt,-5pt)--++(-5pt,-5pt)--++(-5pt,5pt);
\draw [fill=ududff] (1.6941,-1.15245) ++(-5pt,0 pt) -- ++(5pt,5pt)--++(5pt,-5pt)--++(-5pt,-5pt)--++(-5pt,5pt);
\draw [fill=ududff] (2.9646,-2.38665) ++(-5pt,0 pt) -- ++(5pt,5pt)--++(5pt,-5pt)--++(-5pt,-5pt)--++(-5pt,5pt);
\draw [fill=ffttww]  (6,-2) ++(-5pt,0 pt) -- ++(5pt,5pt)--++(5pt,-5pt)--++(-5pt,-5pt)--++(-5pt,5pt);
\draw [fill=ffttww] (6,0) ++(-5pt,0 pt) -- ++(5pt,5pt)--++(5pt,-5pt)--++(-5pt,-5pt)--++(-5pt,5pt);
\draw [fill=ffttww] (4.707,-1.47915) ++(-5pt,0 pt) -- ++(5pt,5pt)--++(5pt,-5pt)--++(-5pt,-5pt)--++(-5pt,5pt);
\draw [fill=ududff] (3.5454,-0.86205) node (5pt) {};
\draw [fill=ududff] (2.6742,-0.09975) node (5pt) {};
\draw [fill=ududff] (1.6941,-1.15245) node (5pt) {};
\draw [fill=ududff] (2.9646,-2.38665) node (5pt) {};
\draw [fill=ffttww] (6,0) node (5pt) {};
\draw [fill=ffttww] (4.707,-1.47915) node (5pt) {};
\draw [fill=ffttww] (6,-2) node (5pt) {};
\draw [fill=ffttww] (8,-2) circle (5pt);
\draw [fill=ffttww] (8,0) circle (5pt);
\draw [fill=ffttww] (8.9541,-0.24495) circle (5pt);
\draw [fill=ffttww] (10,0) circle (5pt);
\draw [fill=ffttww] (10.0068,-0.68055) circle (5pt);
\draw [fill=ffttww] (8.8089,-1.58805) circle (5pt);
\draw [fill=ffttww] (10,-1.37025) circle (5pt);
\draw [fill=ffttww] (10,-2) circle (5pt);
\draw [fill=ududff] (0,-1) circle (5pt);
\draw [fill=ududff] (-2,-0.89835) circle (5pt);
\draw [fill=ududff] (-2,0) circle (5pt);
\draw [fill=ududff] (-2,-2) circle (5pt);
\end{scriptsize}
\end{tikzpicture}
    \caption{A toy graph $G$ %
    partitioned into a monophonic halfspace $H$ (in blue) and its complement $\cH$ (in orange). The diamond-shaped vertices form the boundaries $\partial H$ and $\partial\cH$ %
    and are cliques, see \Cref{mconv-recall}(1).}
    \label{fig:halfspace}
\end{figure}
A set $P$ is an \emph{induced path} if $G[P]$ is a path.
The \emph{length} of a path $P$ is the number  of
edges in $P$. A \emph{shortest $(u,v)$-path} is a $(u,v)$-path with a minimum number of edges. %
The \emph{distance} $d_G(u,v)$ between $u$ and $v$ is the length of a $(u,v)$-geodesic. If there is no ambiguity, we will write $d(u,v)=d_G(u,v)$.
We denote by $\diam(G)=\max\{ d(u,v): u,v\in V\}$  the \emph{diameter} of $G$. %
If $a,b \in V$ and $S \subseteq V$, then $S$ is an $(a,b)$-\emph{separator} if in $G$ every path between $a$ and $b$ intersects $S$. The \emph{boundary} $\partial A$ of a set $A\subseteq V$ is  the set of vertices $u\in A$ having a neighbor in $A^{\scriptstyle c}=V\setminus A$. %

\paragraph{Convexity spaces.}
Let $V$ be a finite set and $\C \subseteq 2^V$. Then $(V,\C)$ is a \emph{convexity space} if (i) $\varnothing, V \in \C$ and (ii) $\bigcap_{C\in\C'} C\in\C$ for every $\C'\subseteq\C$. %
A set $C \subseteq V$ is  \emph{convex} if $C \in \C$. Convexity spaces abstract standard Euclidean convexity (see, e.g., \citealt{van1993theory}). The \emph{convex hull} of $A \subseteq V$ is $\conv(A) = \bigcap_{C\in\C:A\subseteq C} C\in\C$.  %
A set $A$ is called  \emph{ $c$-independent} if $\conv(A)\setminus (\bigcup_{a\in A} \conv(A\setminus \{ a\})\ne\varnothing$ and the \emph{Carath\'eodory  number}  is the size of a largest $c$-independent set. A \emph{hull set} is any set $A$ with the property $\conv(A)=V$ and the \emph{hull number}  of $(V,\C)$ is the size of the smallest hull set.
A \emph{halfspace} of $(V,\C)$  is a convex set $H$ with convex complement $\cH=V\setminus H$. %
Two sets $A,B$ are \emph{separable by halfspaces} (or simply, \emph{separable}) if there exists a halfspace $H$ such that $A\subseteq H$ and $B\subseteq \cH$.
We will use the following well-known separation axioms from convexity theory \citep{van1993theory} to prove the tightness of some of our bounds: The space $(V,\C)$ is $S_2$ if any two distinct $p,q\in V$ are separable by halfspaces and $(V,\C)$ is $S_3$ if any convex set $A$ and any $p\notin A$ are separable by halfspaces.
Given two  sets $A$ and $B$ of a convexity space $(X,\C)$, the \emph{shadow of $A$ with respect to $B$}
is the set $A/B=\{ x\in X: \conv(B\cup \{ x\})\cap A\ne \varnothing\}$
\citep{chepoi1986some,chepoi1994separation}.
For $a,b\in X$, we call the shadow $a/b=\{a\}/\{b\}$ a \emph{point-shadow}.
A map $I\colon V\times V \rightarrow 2^V$ is an \emph{interval map} if (i) $u,v\in I(u,v)$ and (ii) $I(u,v)=I(v,u)$ for all $u,v\in V$. An example is the \emph{geodesic interval} of a metric space $(V,d)$ defined by $I_d(u,v) = \{z\in V: d(u,z)+d(z,v)=d(u,v)\}$. Any interval map $I$ defines a convexity space $(V,\C)$ where $C \in \C$ if and only if $I(u,v)\subseteq C$ for all $u,v\in C$. A \emph{graph convexity space} for $G=(V,E)$ is a convexity space $(V,\C)$ where $G[A]$ is connected for all $A\in\C$~\citep{duchet1983ensemble, pelayo2013geodesic}. We denote by %
$c(G)$ the  Carath\'eodory  number and by $\hull(G)$ the hull number of $(V,\C)$.

\paragraph{Monophonic convexity.}
For  vertices $u,v \in V$, the \emph{monophonic interval} (\emph{m-interval}) $\MonInt_m(u,v)$ is the set of all vertices that lie on induced $(u,v)$-paths.
A set $A \subseteq V$ is \emph{monophonically convex} (\emph{m-convex}) if $\MonInt_m(u,v) \subseteq A$ for all $u,v \in A$. The \emph{monophonic convex hull} (\emph{m-hull}) of $S \subseteq V$, denoted $\cm(S)$, is the smallest m-convex set containing $S$. Let $\cM=\cM(G)$ denote the set of  m-convex sets of a graph $G$.
We denote by $\Hm(G)$ the set of  of all monophonic halfspaces of a graph $G$, i.e., all halfspaces of the convexity space $(V,\cM)$. For an edge
 $ab$ of $G$, let $\Hm(ab)$ be the set of all  $H\in \Hm(G)$ with $a\in H$ and $b\in\cH$. %
For $i\in\{2,3,4\}$ we say that a graph is $S_i$ if the convexity space given by monophonically convex sets of $G$ satisfies the $S_i$ separation axiom.
The monophonic convexity has been investigated in several papers. \cite{farber1986convexity} and  \cite{duchet1988convex} investigated the Helly, Radon, and Carathéodory numbers of m-convexity and \cite{bandelt1989graphs} characterized the graphs satisfying %
various separation by m-halfspaces properties.  %
The complexity of various algorithmic problems
for m-convexity was investigated by \cite{dourado2010complexity} and \cite{gonzalez2020covering}.
We recall some structural and complexity results. %
\begin{lemma}  \label{mconv-recall} Let $G=(V,E)$ be a graph and $S\subseteq V$.
\begin{enumerate}[(1)]
\item  $S\subseteq V$  is m-convex iff for each connected component $B$ of $G[S^{\scriptstyle c}]$,  $\partial_B(S)$ is a clique %
\citep{duchet1988convex}.
\item  $c(G)\le 2$, i.e., $\cm(S)=\bigcup_{u,v\in S} \cm(u,v)$. The graph $G$ is complete if and only if $c(G)=1$ \citep{duchet1988convex,farber1986convexity}.
\item $r(G)=\omega(G)$ if $\omega(G)\ge 3$ and $r(G)\le 3$
if $\omega(G)\le 2$ \citep{duchet1988convex}.
\item m-hulls (and thus shadows) are computable in polynomial time \citep{dourado2010complexity}.
\item a minimum hull set can be computed in polynomial time \citep{dourado2010complexity}.
\end{enumerate}
\end{lemma}

\paragraph{Hypothesis spaces and sign vectors.}
In machine learning, a \emph{hypothesis space} on a set $V$ is any collection $\CC$ of subsets of $V$.
The \emph{VC-dimension} $\vcdim(\mathcal{C})$ of $\mathcal{C}$ is the size of a largest set
$S\subseteq V$ \emph{shattered} by $\mathcal{C}$, that is, such that $\{C\cap S: C\in \mathcal{C}\}=2^S$.
A \emph{sample} is a set $X=\{(x_1,y_1),\ldots,(x_m,y_m)\}$, where $x_i\in V$ and $y_i\in\{-1,+1\}$.
A sample $X$ is \emph{realizable by a hypothesis} $C\in \CC$  if $y_i=+1$ when $x_i\in C$, and $y_i=-1$ when $x_i\notin C$;  we also say that the hypothesis $C$ is \emph{consistent} (for $X$). %
A sample $X$ is  a \emph{realizable sample} of $\CC$ if $X$ is realizable  by some hypothesis $C$ in $\CC$. We denote by $\Sam(\CC)$ the set of all realizable samples. To encode hypotheses of $\CC$ and samples of $\Sam(\CC)$, we use the language of  sign vectors
from oriented matroids theory~\citep{BLSWZ,ChChMcRaVa}. Let $\covectors$ be a non-empty \textit{set of sign vectors}, \textit{i.e.},
maps from $V$ to $\{\pm 1,0\} = \{-1,0,+1\}$.
For $X \in \covectors$, let $X^+ = \{v\in V: X_v=+1\}$ and $X^-= \{v\in V:
X_v=-1\}$. The set $\underline{X} = X^+ \cup X^-$
is called the \emph{support} of $X$.
We denote by $\preceq$  the product ordering
on $\{ \pm 1,0\}^V$ relative to the ordering of signs with $0 \preceq -1$  and $0 \preceq +1$.
Any hypothesis space $\CC\subseteq 2^V$ can be viewed as a set of sign vectors of $\{ \pm 1\}^V$:
each hypothesis $C\in \CC$ is encoded by the sign vector $X(C)$, where $X_v(C)=+1$ if $v\in C$ and
$X_v(C)=-1$ if  $v\notin C$. In what follows, we consider a hypothesis space $\CC$ simultaneously as a
collection of sets and as a set of  $\{ \pm 1\}$-vectors.
Note that each sample $X$ is just a sign vector and that the samples realizable by a hypothesis $C\in \CC$ are all $X\in \{ \pm 1,0\}^V$ such that $X\preceq C$.
In this paper, as  $\CC$ we will consider the set $\Hm(G)$ of monophonic halfspaces of a graph $G=(V,E)$ and we set $d=\vcdim(\Hm(G))$.

\section{Monophonic halfspace separation}\label{sec:hsp}\label{sec:polycheck}
This section describes our first main technical contribution: a polynomial-time algorithm for monophonic halfspace separation. 

\begin{definition}[Halfspace separation problem, \citealt{seiffarth2023maximal}] Given a convexity space $(V,\C)$ and two subsets $A,B\subseteq V$, the halfspace separation problem asks to decide whether $A$ and $B$ are separable by complementary halfspaces $H,\cH \in \C$.
\end{definition}
The halfspace separation problem is an instance of the \emph{consistent hypothesis finding} problem. In Euclidean spaces, it is equivalent to linear separability
\citep{BoGuVa,MiPa}. \citet{seiffarth2023maximal} introduced this problem in general convexity spaces and showed that it is NP-complete already for geodesic convexity in graphs. %
In contrast, we prove the following result. %
\begin{theorem}[Poly-time separation of m-halfspaces]
\label{thm:monophonic-halfspace-separation}
 There is an algorithm that, given a graph $G=(V,E)$ and two subsets $A,B \subseteq V$, in polynomial time either outputs $H \in \Hm(G)$ that separates $A$ and $B$ or correctly decides that no such $H$ exists.
\end{theorem}
As a consequence of \Cref{thm:monophonic-halfspace-separation}, by testing whether $A=\{u\}$ and $B=\{v\}$ are separable for every $u,v\in V$, we obtain:
\begin{corollary}
In polynomial time one can decide if $V(G)$ admits a proper bipartition into m-convex sets in~$G$ (that is, if there exists $H\in\Hm(G)$ such that $\varnothing\subsetneq H \subsetneq V$).
\end{corollary}

This answers an open question by \citet{gonzalez2020covering}, who proved that finding a partition into $k \ge 2$ m-convex sets is NP-hard for $k \ge 3$, leaving open the case $k=2$.

\begin{algorithm}[h]
\caption{\HalfspaceSeparation}
\label{algorithm-sep}
\label{alg:halfspaceseparation}
\LinesNumbered
\DontPrintSemicolon
\SetAlgoNoEnd
\KwInput{a graph $G=(V,E)$ and two nonempty subsets $A,B\subseteq V$}
\KwOutput{a halfspace $H \in \Hm(G)$ separating $A$ and $B$ if one exists, \NO\ otherwise}
 compute a shortest $A$-$B$ path $P=(u_1,\ldots,u_k)$\;
 \For{each edge $u_iu_j \in P$}{
    $(A^*,B^*)=\ShadowClosure(A \cup \{u_i\},B \cup \{u_{i+1}\})$ \;
    \If{$A^*\cap B^* = \varnothing$ \label{alg:halfspaceseparation:if}}{
        compute $R = V\setminus (A^*\cup B^*)$, its boundary $\partial R$, and $S_x$ for every $x \in R$\; \label{algoline:X_P_Q}
        construct the 2-SAT formula $\Phi$ described in \cref{sub:2SATred}\;
        \If{$\Phi$ has a satisfying assignment $\alpha : R \to \{0,1\}$}{
            \textbf{return} $A \cup \{x \in R : \alpha(a_x)=1\}$ \label{algoline:return_H}
        }
    }
}
\textbf{return} \NO\;
\end{algorithm}

The pseudocode of \HalfspaceSeparation, the algorithm behind \Cref{thm:monophonic-halfspace-separation}, is given in \Cref{algorithm-sep}; the rest of this section describes the algorithm and proves the theorem itself.
Intuitively, the goal of \HalfspaceSeparation\ is to cast the problem of finding a halfspace $H \in \Hm(G)$ that separates $A$ and $B$ as a 2-SAT instance.
This could be done by associating a boolean variable $x$ to every vertex in $V$, and expressing the convexity of $H$ through constraints in the form ``if $x,y$ are both in $H$, and $z \in \cm(\{x,y\})$, then $z \in H$'', too.
Unfortunately, doing so for arbitrary $x,y,z$ would require \emph{three} literals per clause, yielding a 3-SAT formula.
We prove however that, after a certain preprocessing (see \Cref{sub:shadowclos}), the constraints of the remaining instance can indeed be expressed using a 2-SAT formula (see \Cref{sub:2SATred}).

\subsection{The shadow closure}\label{sub:shadowclos}
The first step of our algorithm finds those vertices that, for every pair of complementary monophonic halfspaces, can be ``easily'' seen to lie on the same side as $A$ (or $B$).
To this end we shall use the notion of \emph{shadow closure}.

Let $G=(V,E)$ and $A,B \subseteq V$ be the input to the halfspace separation problem.
Recall that we can assume $G$ is connected.
If there exists $H \in \Hm(G)$ that separates $A$ and $B$, then there exists an edge $uv \in E(G)$ such that $A \cup \{u\} \subseteq H$ and $B \cup \{v\} \subseteq H^c$.
In fact, let $P=u_1,\ldots,u_k$ be a shortest path between any two vertices $a\in A$ and $b \in B$.
It is immediate to see that $H \in \Hm(G)$ separates $A$ and $B$ if and only if $H \in \Hm(G)$ separates $A \cup \{u_i\}$ and $B \cup \{u_{i+1}\}$ for some $i$.
Now suppose that there exists $H \in \Hm(G)$ that separates $\cm(A/B)$ and $\cm(B/A)$.
Since $\cm(A/B) \supseteq A/B \supseteq A$ and $\cm(B/A) \supseteq B/A \supseteq B$, then $H$ separates $A$ and $B$, too.
It is not hard to prove that the converse holds as well: 
\begin{lemma}[\citealt{chepoi2024separation}]
Let $A,B \subseteq V$ and $H \in \Hm(G)$. If $A\subseteq H$ and $B\subseteq \cH$, then $\cm(A/B)\subseteq H$ and $\cm(B/A)\subseteq \cH$. 
\end{lemma} 
We conclude that $H$ separates $A$ and $B$ \emph{if and only if} $H$ separates $\cm(A/B)$ and $\cm(B/A)$.

Now we iterate the argument.
Let $A^0=A$ and $B^0=B$, and for every $i=1,2,\ldots$ let $A^i = \cm(A^{i-1}/B^{i-1})$ and $B^i = \cm(B^{i-1}/A^{i-1})$.
By the argument above, $H$ separates $A^{i-1}$ and $B^{i-1}$ if and only if it separates $A^i$ and $B^i$.
Moreover $A^{i-1} \subseteq A^i$ and $B^{i-1} \subseteq B^i$, thus the sequences $(A_i)_{i \ge 1}$ and $(B_i)_{i \ge 1}$ converge to some sets $A^*$ and $B^*$.
We conclude that $H$ separates $A$ and $B$ if and only if it separates $A^*$ and $B^*$.
We call the pair of sets $(A^*,B^*)$ the \emph{shadow closure} of $(A,B)$; and we say that a pair of sets $(A,B)$ is \emph{shadow-closed} whenever $A=\cm(A/B)$ and $B=\cm(B/A)$.
Since $\cm$ can be computed in polynomial time, see \Cref{mconv-recall}, then the iterative procedure above (formalized as \ShadowClosure\ in \Cref{alg:shadowclosure}) computes the shadows closure in polynomial time.
\begin{algorithm}[h]
\label{algorithm-shadowclosure}
\LinesNumbered
\DontPrintSemicolon
\SetAlgoNoEnd
\KwInput{a graph $G=(V,E)$ and two nonempty subsets $A,B\subseteq V$}
\KwOutput{the shadow closure of $(A,B)$}
\While{$A \ne A/B$ \textbf{or}  $B\ne B/A$}{
$A= \cm(A/B)$ and $B= \cm(B/A)$\;
}
\textbf{return} $(A,B)$\;
\caption{\ShadowClosure}
\label{alg:shadowclosure}
\end{algorithm}

By the discussion above, we need only decide in polynomial time, for each $i=1,\ldots,k-1$, whether the shadow closure of $(A \cup \{u_i\},B \cup\{u_{i+1}\})$ is separable by a monophonic halfspace.
This is what \HalfspaceSeparation\ does at lines \ref{algoline:X_P_Q}, as we shall prove in the next section.

\subsection{A reduction to 2-SAT}
\label{sub:2SATred}
Let $A,B \subseteq V$.
We say $A,B$ \emph{osculate} if they are disjoint and adjacent; that is, if $A\cap B = \varnothing$ and $u \sim v$ for some $u \in A$ and $v \in B$.
Note that, if $A^*,B^*$ satisfy the condition at line \ref{alg:halfspaceseparation:if} of \HalfspaceSeparation, then $(A^*,B^*)$ is a shadow-closed osculating pair.
Let then $(A,B)$ be a shadow-closed osculating pair of m-convex sets of $G$. The \emph{boundary of $A$ with respect to $B$} $\partial_B A$ is the set of all  $a\in A$ having a neighbor $b\in B$. Call the set $R=V\setminus (A\cup B)$ the \emph{residue} and let $\partial R$ be its boundary.
Clearly, each pair $(H,\cH)$ of halfspaces separating $A,B$ has the form $H=A\cup A^+, \cH=B\cup B^+$ for a partition $(A^+,B^+)$  of $R$.
To compute this partition $(A^+,B^+)$ of $R$ we define a 2-SAT formula $\Phi'$ that can be constructed in polynomial time.
We then define a refined version, $\Phi$.
The two formulas are equivalent; the advantage of $\Phi'$ is its simplicity, and the advantage of $\Phi$ will be apparent in the proofs of \Cref{sec:decomposition-VC-dim}.

Both $\Phi'$ and $\Phi$ have one boolean variable $a_x$ for every $x \in R$.
The meaning is that $a_x$ indicates whether $x$ is on the same side of $H$ as $A$.
The formulas contain different types of constraints.
First, for each $x\in R$ let $S_x$ be the set of all $x_0\in \partial R$ such that there exist $a\in \partial_B A,b\in \partial_A B$ with $a\sim b$ and $x_0\in \cm(x,a)\cap \cm(x,b)$. 
The first two types of constraints are called \emph{equality constraints} and \emph{difference constraints} and are defined by setting: 
\begin{align}
    a_x&=a_{x_0} \text{ for any } x\in  R \text{ and } x_0\in S_x\,,\label{eq:1satconstraints} \\
      a_{x_0}&\ne a_{y_0} \text{ if } x_0,y_0\in \partial R \text{ and } x_0 \nsim y_0\,\label{eq:11satconstraints}.
\end{align}
The third set consists of \emph{implication constraints} and is defined by setting: 
\begin{equation}\begin{aligned}\label{eq:33satconstraints}
    a_x\twoheadrightarrow_A a_{y} \text{ for any } x, y\in R  \text{ such that } y\in \cm(x,z)  \text{ for some } z\in A \,,\\
    a_x\twoheadrightarrow_B a_{y} \text{ for any } x,y\in R  \text{ such that }  y\in \cm(x,z)  \text{ for some } z\in B \,.\\
\end{aligned}\end{equation}
We also consider a particular subset of implication constraints for $x,y\in \partial R$.
A vertex $x \in R$ is said to $A$-\emph{imply} $y$ (notation $x\rightarrow_A y$) if there exists an induced path $(x,\ldots,y,z)$ with $z\in A$. Analogously, $x$ is said to $B$-\emph{imply} $y$ (notation $x\rightarrow_B y$)
if there exists an induced path $(x,\ldots,y,z)$ with $z\in B$.  
We consider the following constraints:
\begin{equation}\begin{aligned}\label{eq:3satconstraints}
    a_x\rightarrow_A a_{y} \text{ for any } x, y\in \partial R  \text{ such that } x\rightarrow_A y\,,\\
    a_x\rightarrow_B a_{y} \text{ for any } x,y\in \partial R  \text{ such that } x\rightarrow_B y\,.\\
\end{aligned}\end{equation}
Clearly, if $a_x\rightarrow_A a_{y}$, then $a_x\twoheadrightarrow_A a_{y}$. 
All pairs $a_x\rightarrow_A a_y$ can be constructed in polynomial time.
For this, pick any $z\in A$ adjacent to $y$, and test whether $y$ and $x$ can be connected by a path (and thus by an induced path) in the subgraph $G[(R\setminus N(z))\cup \{ y\}]$.
Note that $x\rightarrow_A y$ if and only if some $z$ passes the test.

It is now easy to see that, in polynomial time, one can cast all constraints defined above as 2-SAT formulas, as follows:
\begin{align}
    \Phi_1 &= \bigwedge_{x \in R, x_0 \in S_x} (\overline{a_x} \vee a_{y}) \wedge ({a_x} \vee \overline{a_y}) 
    \\
    \Phi_2 &= \bigwedge_{\substack{x,y \in \partial R\\ x \nsim y}} (a_x\vee a_y) \wedge (\overline{a_x}\vee \overline{a_y})
    \\
    \Phi_3 &= \bigwedge_{\substack{x,y \in R, z \in A \\ y\in\cm(x,z)}} (\overline{a_x} \vee a_{y}) \,\,\, \wedge \bigwedge_{\substack{x,y \in R, z \in B \\ y\in\cm(x,z)}} ({a_x} \vee \overline{a_y})
    \\
    \Phi_4 &=  \bigwedge_{\substack{x,y \in \partial R\\ x \rightarrow_A y}} (\overline{a_x} \vee a_{y}) \,\,\, \wedge \bigwedge_{\substack{x,y \in \partial R \\ x \rightarrow_B y}} ({a_x} \vee \overline{a_y})
\end{align}

We denote by $\Phi'$ the 2-SAT formula consisting of all clauses corresponding to difference constraints (\ref{eq:11satconstraints}) and implication constraints (\ref{eq:33satconstraints}).
Formally,
\begin{align}
    \Phi' = \Phi_2 \wedge \Phi_3 \,. \label{eq:Phi'}
\end{align}
We denote by $\Phi$ the 2-SAT formula consisting of all clauses corresponding to equality constraints (\ref{eq:1satconstraints}),  difference constraints (\ref{eq:11satconstraints}), and implication constraints (\ref{eq:3satconstraints}). Formally,
\begin{align}
    \Phi &= \Phi_1 \wedge \Phi_2 \wedge \Phi_4 \,. \label{eq:Phi}
\end{align}

Both  $\Phi'$ and $\Phi$ can be constructed in polynomial time and can be solved in linear time \citep{AsPlTa}. 
Our main result is that the solutions of $\Phi'$ and $\Phi$ precisely characterize the m-halfspaces $(H,\cH)$ separating $A$ and $B$.
This also completes the proof of Theorem~\ref{thm:monophonic-halfspace-separation}.
\begin{theorem}\label{l:halfspaces-from-SAT} Let $(A,B)$ be a pair of osculating shadow-closed m-convex sets of $G$ and $R=V\setminus (A\cup B)$. For a partition $(A^+,B^+)$ of $R$, the following conditions are equivalent: 
\begin{enumerate}\itemsep0pt
\item[(i)] $H=A\cup A^+, \cH=B\cup B^+$ are complementary m-halfspaces of $G$;
\item[(ii)] $A^+=\{ x\in R: \alpha(a_x)=1\}$ and $B^+=\{ x\in R: \alpha(a_x)=0\}$ for a solution $\alpha$ of $\Phi'$;
\item[(iii)] $A^+=\{ x\in R: \alpha(a_x)=1\}$ and $B^+=\{ x\in R: \alpha(a_x)=0\}$ for a solution $\alpha$ of $\Phi$.
\end{enumerate}
\end{theorem}
To prove \cref{l:halfspaces-from-SAT} we need the following technical lemma.
\begin{lemma} \label{x_0} Any $x_0\in \partial R$ is adjacent to all vertices of $\partial_B A\cup \partial_A B$. Consequently, for any $x\in R$, the set $S_x$ is nonempty and $S_x$ separates $x$ from any vertex of $A\cup B$. 
\end{lemma}
\begin{proof} Suppose  $x_0$ is adjacent to $z\in A\cup B$, say $z\in A$. Pick $a\in \partial_B A$ and $b\in \partial_A B$ with $a\sim b$.
Suppose that $x_0\nsim b$ and assume that among all neighbors of $x_0$ in $A$, $z$ is closest to $a$.  Let $P'$ be a shortest $(z,a)$-path (if $x_0\sim a$, then $z=a$). 
Since $A$ is m-convex, $P'\subseteq A$. From the choice of $z$ and since $x_0\nsim b$, the $(x_0,b)$-path $P''$ consisting of $P'$ plus the edges $zx_0$ and $ab$ is induced. 
This implies that $z\in A/b\subseteq A/B$, contradicting that $(A,B)$ is shadow-closed and $x_0\in R$. Hence $x_0\sim b$. If $x_0\nsim a$, then $x_0\in b/a\subseteq B/A$, contradicting again 
that the pair $(A,B)$ is shadow-closed.  This establishes the first assertion. Since any induced path $P$ from $x\in R$ to a vertex of $A\cup B$ contains a vertex $x_0\in \partial R$, from the first assertion we get $P\cap S_x\ne\varnothing$ and thus $S_x$ is nonempty and separates $x$ from $A\cup B$. 
\end{proof}
\vspace*{-1em}
\begin{proof}[of \Cref{l:halfspaces-from-SAT}] To prove (i)$\Rightarrow$(ii), let $(H,\cH)$ be a pair of complementary halfspaces with $A\subseteq H$ and $B\subseteq \cH$. Let $A^+=H\cap R$ and $B^+=\cH\cap R$.
Note that by Lemma \ref{x_0} both $H\cap \partial R$ and $\cH\cap \partial R$ are cliques, so the difference constraints (\ref{eq:11satconstraints}) hold. If 
$x\in A^+$ and $y\in \cm(x,z)\cap R$ for some $z\in A$, then the m-convexity of $H$ implies that $y\in H$, thus $y\in A^+$. Thus the implication constraints (\ref{eq:33satconstraints}) hold too. 

Now, we show (ii)$\Rightarrow$(iii). The equality constraints  (\ref{eq:1satconstraints}) and the implication constrains (\ref{eq:3satconstraints}) are particular implication constrains from (\ref{eq:33satconstraints}). Consequently,  $\Phi$ is a sub-formula of $\Phi'$, and thus each solution of $\Phi'$ is a solution of $\Phi$. 

Finally, to show that (iii)$\Rightarrow$(i), pick any satisfying assignment $\alpha$ of $\Phi$. Let $A^+=\{ x\in R: \alpha(a_x)=1\}, B^+=\{ x\in R: \alpha(a_x)=0\}$ and set $H=A\cup A^+, \cH=B\cup B^+$.  
Suppose by way of contradiction that $H$ is not m-convex. Then $H$ contains two non-adjacent vertices
$x,x'$ which can be connected outside $H$ by an induced path $P$. We can suppose without loss of generality that all vertices of $P\setminus \{ x,x'\}$ belong to $\cH$.  Since $H=A\cup A^+$ and $A$ is m-convex, either $x,x'\in A^+$ or $x\in A^+$ and $x'\in A$. 

First, let  $x\in A^+$ and $x'\in A$. If $x\in R\setminus \partial R$, by Lemma~\ref{x_0}, $P$ intersect $S_x$ in a vertex $x_0$. 
By equality constraints (\ref{eq:1satconstraints}), $x_0$ belongs to $H$, contrary to the assumption that $P\setminus \{ x,x'\}\subseteq \cH$. Now, let $x\in \partial R$.
Let $y$ be the neighbor of $x'$ in $P$. Then $y\in P\setminus \{ x,x'\}\subseteq \cH=B\cup B^+$. If $y\in B$, then $x'\in \partial_B A$ and by Lemma \ref{x_0}
we get that $x\sim x'$, which is impossible. Therefore $y\in R$. Since $y\sim x'$ and $x'\in A$, necessarily $y\in \partial R$. Consequently,  $x\rightarrow_A y$.
By conditions (\ref{eq:3satconstraints}),  $a_x=a_{y}=1$, contrary to the assumption that $y\in \cH$.

Now, let $x,x'\in A^+$. We assert that $x,x'\in \partial R$.  Let $z$ be the neighbor of $x$ in $P$ and $z'$ be the neighbor of $x'$ in $P$. First we show that $z,z'\in B\cup \partial R$. Pick any $a\in \partial_B A$ and $b\in \partial_A B$ with $a\sim b$. Since the pair $(A,B)$ is shadow-closed, $d(x,a)=d(x,b)$. Suppose by way of contradiction that $z\notin B$ and $z$ is not adjacent to $a$ and $b$ (by Lemma \ref{x_0}, $z$ is adjacent to $a$ if and only if $z$ is adjacent to $b$). 
Pick any shortest path $Q$ from $a$ to $x$ and let $x_0$ be the neighbor of $a$ in $Q$. From Lemma \ref{x_0}, $x_0$ is also adjacent to $b$. Since $d(x,a)=d(x,b)$,  $x_0\in \cm(x,a)\cap \cm(x,b)$, thus $x_0\in S_x$.  Furthermore, $x_0\in S_q$ for any vertex $q\in Q'=Q\setminus \{ a\}$. From the equality constraints (\ref{eq:1satconstraints}), we deduce that all vertices of $Q'$ belong to $H$. Consequently, $z$ does not belong to $Q'$.  %
Since $z\nsim a$, from the path consisting of the edge $zx$, the path $Q'$, and the edge $x_0a$ we can extract an induced
$(z,a)$-path passing via $x_0$. Analogously, since $z\nsim b$, from the path consisting of the edge $zx$, the path $Q$, and the edge $x_0b$ we can extract an induced $(z,b)$-path passing via $x_0$.
Consequently, $x_0\in \cm(z,a)\cap \cm(z,b)$, showing that $x_0\in S_z$. By (\ref{eq:1satconstraints}), we get $a_z=a_{x_0}=a_x$, thus $z$ belongs to $H$, contrary to our assumption. Thus
$z,z'\in B\cup \partial R$. If $z\in B$, then $x\in \partial R$ by the definition of $\partial R$. %
If $z\in \partial R$,
then $z\sim a,b$. If $x\notin \partial R$, then $x\nsim a,b$, thus  $(x,z,a)$ and $(x,z,b)$ are induced paths, yielding $z\in S_x$. By (\ref{eq:1satconstraints}) we get $a_z=a_x$, contrary to $z\in \cH$. 
Therefore, $x,x'\in \partial R$. %
Since $x,x'\in H\cap \partial R$ from  (\ref{eq:11satconstraints}), we get $x\sim x'$ and thus $P=(x,x')$. This concludes the proof.
\end{proof}

\begin{remark} By Theorem \ref{l:halfspaces-from-SAT}, both $\Phi'$ and $\Phi$ have the same set of solutions. The advantage of $\Phi'$ is its simplicity. The advantage of $\Phi$, which is a subformula of $\Phi'$ and has refined constraints, is to highlight the structure of the solution space---this is the subject of \cref{sec:decomposition-VC-dim}. 
\end{remark}

From Theorem~\ref{l:halfspaces-from-SAT} we obtain the following
corollary.
\begin{corollary} \label{c:bijection-halfspaces-assignments} Let $(A,B)$ be a pair of osculating shadow-closed m-convex sets of $G$. 
The number of halfspaces $H$ with $A\subseteq H$ and $B\subseteq \cH$ equals the number of satisfying assignments of the 2-SAT formula $\Phi$ (or, equivalently, $\Phi'$).
\end{corollary}

\section{The structure of monophonic halfspaces} \label{sec:decomposition-VC-dim}
This section investigates the structure of m-halfspaces.
Our main result is a \emph{decomposition theorem} saying that every m-halfspace $H$ that separates an osculating shadow-closed pair $(A,B)$ %
can be written in a nice form.
More precisely, every such $H$ can be written as a disjoint union of certain subsets, called \emph{cells}.
This characterization yields efficient algorithms for all the learning settings below, including empirical risk minimization and compression, as well as tight bounds on the VC-dimension of the class of halfspaces separating $A$ and $B$.

For an edge $ab$  of $G$, recall that $\Hm(ab)=\{H\in\Hm(G): a\in H, b\in \cH \}$ is the set of m-halfspaces $H \in \Hm(G)$ that contain $a$ and avoid $b$. %
We show that all such m-halfspaces can be represented in a specific form.
More precisely, we describe a family of disjoint non-empty subsets of vertices $\{C_1,\dots,C_k\}$, each of which can in turn be partitioned into two, as $C_i=C'_i\mathbin{\dot{\cup}} C''_i$
(one of two sets $C_i'$, $C_i''$ can be empty).  %
We additionally partition $\{C_1,\dots,C_k\}$ into $\scC_1,\dots,\scC_p$ and denote $\scC_{< \ell} = \scC_1\cup\cdots\cup \scC_{\ell-1}$ for $\ell\in[p]$, similarly we define $\scC_{> \ell}$.
We call this structure $(\{C_1',C_1'',\dots,C_k',C_k''\}, \scC_1,\dots,\scC_p)$ the \emph{cell decomposition} of $\Hm(ab)$ (see \Cref{section:cells} for details).
Our cell decomposition theorem, stated next, says that $H \in \Hm(ab)$ if and only if $H$ can be expressed in a certain form using the cell decomposition.
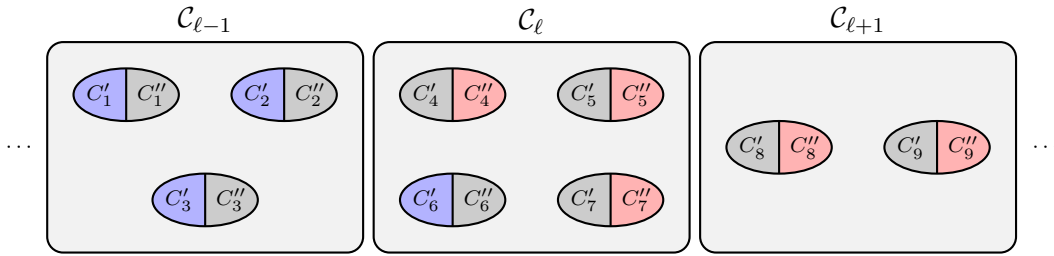
\begin{figure}
    \centering
    \scriptsize
\begin{tikzpicture}[scale=.7]

  \node at (-.5, 2) {$\cdots$};

  \begin{scope}[xshift=0cm]
    \draw[fill=gray!10, draw=black, line width=0.8pt, rounded corners=6pt]
      (0,0) rectangle (6,4);
    \node at (3,4) [anchor=south, font=\normalsize] {$\scC_{\ell-1}$};

    \foreach \cx/\cy/\i in {1.5/3/1, 4.5/3/2, 3/1/3} {
      \begin{scope}
        \clip (\cx-1,\cy-0.5) rectangle (\cx,\cy+0.5);
        \fill[blue!30] (\cx,\cy) ellipse [x radius=1cm, y radius=0.5cm];
      \end{scope}
      \begin{scope}
        \clip (\cx,\cy-0.5) rectangle (\cx+1,\cy+0.5);
        \fill[black!20] (\cx,\cy) ellipse [x radius=1cm, y radius=0.5cm];
      \end{scope}
      \draw[line width=0.8pt] (\cx,\cy) ellipse [x radius=1cm, y radius=0.5cm];
      \draw[line width=0.8pt] (\cx, {\cy - 0.5}) -- (\cx, {\cy + 0.5});
      \node at ({\cx - 0.5}, \cy) {$C_{\i}'$};
      \node at ({\cx + 0.5}, \cy) {$C_{\i}''$};
    }
  \end{scope}

  \begin{scope}[xshift=6.2cm]
    \draw[fill=gray!10, draw=black, line width=0.8pt, rounded corners=6pt]
      (0,0) rectangle (6,4);
    \node at (3,4) [anchor=south, font=\normalsize] {$\scC_{\ell}$};

\begin{scope}
  \clip (1.5-1, 3-0.5) rectangle (1.5, 3+0.5);
  \fill[black!20] (1.5,3) ellipse [x radius=1cm, y radius=0.5cm];
\end{scope}
\begin{scope}
  \clip (1.5, 3-0.5) rectangle (1.5+1, 3+0.5);
  \fill[red!30] (1.5,3) ellipse [x radius=1cm, y radius=0.5cm];
\end{scope}
\draw[line width=0.8pt] (1.5,3) ellipse [x radius=1cm, y radius=0.5cm];
\draw[line width=0.8pt] (1.5, {3 - 0.5}) -- (1.5, {3 + 0.5});
\node at ({1.5 - 0.5}, 3) {$C_{4}'$};
\node at ({1.5 + 0.5}, 3) {$C_{4}''$};

\begin{scope}
  \clip (1.5-1, 1-0.5) rectangle (1.5, 1+0.5);
  \fill[blue!30] (1.5,1) ellipse [x radius=1cm, y radius=0.5cm];
\end{scope}
\begin{scope}
  \clip (1.5, 1-0.5) rectangle (1.5+1, 1+0.5);
  \fill[black!20] (1.5,1) ellipse [x radius=1cm, y radius=0.5cm];
\end{scope}
\draw[line width=0.8pt] (1.5,1) ellipse [x radius=1cm, y radius=0.5cm];
\draw[line width=0.8pt] (1.5, {1 - 0.5}) -- (1.5, {1 + 0.5});
\node at ({1.5 - 0.5}, 1) {$C_{6}'$};
\node at ({1.5 + 0.5}, 1) {$C_{6}''$};

\begin{scope}
  \clip (4.5-1, 3-0.5) rectangle (4.5, 3+0.5);
  \fill[black!20] (4.5,3) ellipse [x radius=1cm, y radius=0.5cm];
\end{scope}
\begin{scope}
  \clip (4.5, 3-0.5) rectangle (4.5+1, 3+0.5);
  \fill[red!30] (4.5,3) ellipse [x radius=1cm, y radius=0.5cm];
\end{scope}
\draw[line width=0.8pt] (4.5,3) ellipse [x radius=1cm, y radius=0.5cm];
\draw[line width=0.8pt] (4.5, {3 - 0.5}) -- (4.5, {3 + 0.5});
\node at ({4.5 - 0.5}, 3) {$C_{5}'$};
\node at ({4.5 + 0.5}, 3) {$C_{5}''$};

\begin{scope}
  \clip (4.5-1, 1-0.5) rectangle (4.5, 1+0.5);
  \fill[black!20] (4.5,1) ellipse [x radius=1cm, y radius=0.5cm];
\end{scope}
\begin{scope}
  \clip (4.5, 1-0.5) rectangle (4.5+1, 1+0.5);
  \fill[red!30] (4.5,1) ellipse [x radius=1cm, y radius=0.5cm];
\end{scope}
\draw[line width=0.8pt] (4.5,1) ellipse [x radius=1cm, y radius=0.5cm];
\draw[line width=0.8pt] (4.5, {1 - 0.5}) -- (4.5, {1 + 0.5});
\node at ({4.5 - 0.5}, 1) {$C_{7}'$};
\node at ({4.5 + 0.5}, 1) {$C_{7}''$};
  \end{scope}

  \begin{scope}[xshift=12.4cm]
    \draw[fill=gray!10, draw=black, line width=0.8pt, rounded corners=6pt]
      (0,0) rectangle (6,4);
    \node at (3,4) [anchor=south, font=\normalsize] {$\scC_{\ell+1}$};

    \foreach \cx/\cy/\i in {1.5/2/8, 4.5/2/9} {
      \begin{scope}
        \clip (\cx-1,\cy-0.5) rectangle (\cx,\cy+0.5);
        \fill[black!20] (\cx,\cy) ellipse [x radius=1cm, y radius=0.5cm];
      \end{scope}
      \begin{scope}
        \clip (\cx,\cy-0.5) rectangle (\cx+1,\cy+0.5);
        \fill[red!30] (\cx,\cy) ellipse [x radius=1cm, y radius=0.5cm];
      \end{scope}
      \draw[line width=0.8pt] (\cx,\cy) ellipse [x radius=1cm, y radius=0.5cm];
      \draw[line width=0.8pt] (\cx, {\cy - 0.5}) -- (\cx, {\cy + 0.5});
      \node at ({\cx - 0.5}, \cy) {$C_{\i}'$};
      \node at ({\cx + 0.5}, \cy) {$C_{\i}''$};
    }
  \end{scope}

  \node at (19, 2) {$\cdots$};

\end{tikzpicture}
    \caption{Visualization of the cell decomposition. Vertices are partitioned into cells (the oval halves). We first choose a block $\scC_\ell$. In this block an arbitrary selection of either $C_i'$ or $C_i''$ for each $i$ is valid. In the blocks before $\scC_\ell$ each first half $C_i'$ must be selected and in blocks after $\scC_\ell$ each second half $C_i''$ must be selected. The union of any such combination of cells yields a valid m-halfspace (see \Cref{thm:unifiedDecomposition}).}
    \label{fig:decomposition}
\end{figure}

\begin{theorem}[Cell decomposition of m-halfspaces, simplified version]\label{thm:unifiedDecomposition}
    Let $G=(V,E)$ be a graph and $ab \in E$ with $\Hm(ab)\neq\varnothing$. There exist a set $A^*\subseteq V$ and a cell decomposition $(\{C_1',C_1'',\dots,C_k',C_k''\}, \scC_1,\dots,\scC_p)$ of $\Hm(ab)$
    such that $H\in\Hm(ab)$ if and only if
        \[
  H=A^* \cup\left(\bigcup\limits_{C_i\in\scC_{< \ell}\,\cup\,\scC_\ell'} C'_i\right)\cup  \left(\bigcup\limits_{C_i\in(\scC_\ell\setminus\scC_\ell')\,\cup\,\scC_{> \ell}}C''_i\right)
  \quad \text{ for some } \ell\in[p] \text{ and }\scC_\ell'\subseteq \scC_\ell\,.
    \]
    Moreover, computing $A^*$ and the cell decomposition requires polynomial time.
\end{theorem}

The details of the decomposition and the full proof are quite technical and require setting up a good deal of notation and definitions, which we start doing in the next sections.

\subsection{A closer look at the 2-SAT formula}
Let $(A,B)=\ShadowClosure(\{ a\},\{ b\})$.
If $A\cap B\ne \varnothing$, then $\Hm(ab)=\varnothing$. 
If instead $A\cap B=\varnothing$, then $(A,B)$ is an osculating shadow-closed pair, see \cref{sub:2SATred}.
Let $\Phi(ab)$ be the 2-SAT formula $\Phi$ as defined in \cref{eq:Phi} for the pair $(A,B)$.
By Corollary~\ref{c:bijection-halfspaces-assignments}, the satisfying assignments of $\Phi(ab)$ are in bijection with the pairs of complementary halfspaces of $\Hm(ab)$. \emph{From now on we assume that $\Phi(ab)$ is satisfiable and thus $\Hm(ab) \ne \varnothing$.}

The result itself, \cref{t:reminders=ideals}, is quite technical, and requires setting up a good deal of notation and definitions, which we start doing next.
Let $(A,B)=\ShadowClosure(\{ a\},\{ b\})$.
If $A\cap B\ne \varnothing$, then $\Hm(ab)=\varnothing$. 
If instead $A\cap B=\varnothing$, then $(A,B)$ is an osculating shadow-closed pair, see \cref{sub:2SATred}.
Let $\Phi(ab)$ be the 2-SAT formula $\Phi$ as defined in \cref{eq:Phi} for the pair $(A,B)$.
By Corollary~\ref{c:bijection-halfspaces-assignments}, the satisfying assignments of $\Phi(ab)$ are in bijection with the pairs of complementary halfspaces of $\Hm(ab)$. \emph{From now on we assume that $\Phi(ab)$ is satisfiable and thus $\Hm(ab) \ne \varnothing$.}
We recall some definitions about 2-SAT, in which we follow \cite{Fe}. Let $\Phi$ be any 2-SAT formula. A variable $a_x$ is \emph{trivial} if it has the same value in all solutions of $\Phi$.
Testing if $a_x$ is trivial takes linear time by substituting $a_x=1$ and $a_x=0$ and checking satisfiability of the resulting formulas.
Two nontrivial variables $a_x$ and $a_y$ are \emph{equivalent} if $a_x=a_y$ in all solutions of $\Phi$, or if $a_x=\overline{a_y}$ in all solutions of $\Phi$. Being equivalent is an equivalence relation, thus the set of variables of $\Phi$ is partitioned into equivalence classes.  
Testing if $a_x$ and $a_y$ are equivalent takes linear time by substituting $a_x, a_y$ with $\{ 0,0\},\{0,1\},\{ 1,0\},\{ 1,1\}$ and testing satisfiability of the resulting formulas.
Each equivalence class is thus partitioned in two \emph{groups}; all variables in a group have the same value in all solutions, and variables in different groups of the same equivalence class have different values in all solutions.
Note also that some groups may be empty.
Clearly, in linear time one can compute the equivalence classes of $\Phi$ and their partitions into groups.

We apply the previous definitions to the 2-SAT formula $\Phi(ab)$ and compute in polynomial time its trivial variables, equivalence classes of variables, and the partition of each equivalence class into two parts.
We can suppose that $\Phi(ab)$ does not contain trivial
variables.  Indeed, if $a_x$ is trivial, say $a_x=1$ in all solutions, then %
$x\in H$ for all $H\in\Hm(ab)$. 
Then we can set $(A,B)= \ShadowClosure(\cm(A\cup \{ x\}),B)$ and define a new 2-SAT formula with respect to the new shadow-closed pair $(A,B)$ and smaller
set $R=V\setminus (A\cup B)$. Under this condition, we have a simple characterization of the implication constraints of \cref{eq:3satconstraints}:

\begin{lemma} \label{l:forgotten} If $x,y\in \partial R$ and $x\rightarrow_A y$ or $x\rightarrow_B y$, then $x\sim y$ in $G$. Furthermore, for $x,y\in \partial R$ with $x\sim y$
we have $x\rightarrow_A y$ (respectively, $x\rightarrow_B y$) if and only if
there exists $u'\in A\setminus \partial_B A$ (respectively, there exists $v'\in B\setminus \partial_A B$) adjacent to $y$ and not adjacent to $x$.
\end{lemma}

\begin{proof} Suppose  $x\rightarrow_A y$ but in $G$ we have $x\nsim y$ . Then by difference constraints (\ref{eq:11satconstraints}), we have $a_x\ne a_{y}$. On the other hand, $x\rightarrow_A y$ and  (\ref{eq:3satconstraints}) imply that if $a_x=1$, then $a_{y}$ is also 1. Therefore there is no satisfying assignment of $\Phi(ab)$ such that $a_x=1$, thus $a_x$ must be a trivial variable of $\Phi(ab)$.
This shows that $x\rightarrow_A y$ implies $x\sim y$. The second assertion follows from the first assertion and the definition of the relations $x\rightarrow_A y$ and $x\rightarrow_B y$.
\end{proof}

\subsection{Cells and the graphs of cells}\label{section:cells}
Let  $\cC=\{C'_1,C''_1,\ldots,C'_k,C''_k\}$  be the partition of the vertices of $R$ defined by the partition of the set $a_x, x\in R$ of variables of $\Phi(ab)$ into groups.
Formally, $x,y\in C'_i$ or $x,y\in C''_i$ if and only if $a_x=a_y$ in all solutions of $\Phi(ab)$ and $x\in C'_i$ and $y\in C''_i$ if and only if $a_x\ne a_y$ in all solutions
of $\Phi(ab)$.  The sets $C'_1,C''_1,\ldots,C'_k,C''_k$ are called \emph{cells}  (some cells may be empty). To simplify the notation, we will denote by $\cC$ also the collection
of nonempty cells. The cells $C'_i,C''_i$ are called \emph{twins} and we set
$\twin(C'_i)=C''_i$ and $\twin(C''_i)=C'_i$ (even if one of the two cells is empty).  Denote by $\cT$ the set of all nonempty cells $C$ such that $\twin(C)\ne\varnothing$. Clearly, if $C\in \cT$, then $\twin(C)\in \cT$. From the definition it follows that
each nonempty cell $C\in \cC$ intersects the set $\partial R$ in a nonempty clique, which we denote by $\partial C=C\cap \partial R$. We will repeatedly use (without explicit reference) the following properties of cells: 

\begin{lemma} \label{l:non-adjacent} Let $C$ be a nonempty cell of $\cC$.
\begin{enumerate}
\item[(1)] $\partial C\ne \varnothing$;
\item[(2)] If $x\in \partial C$, then all vertices $y\in \partial R$ not adjacent to $x$ belong to $\twin(C)$. Consequently, if $C'$ is a nonempty  cell different from $C$ and $\twin(C)$, then all vertices of $\partial C$ are adjacent to all vertices of $\partial C'$;
\item[(3)] $|\cC|\le 2\omega(G)$.
\end{enumerate}
\end{lemma}

\begin{proof}   To prove assertion (1), pick any $x\in C$. By equality constraints (\ref{eq:1satconstraints}),  $S_x\subseteq C$. Since $S_x\ne\varnothing$, $\partial C$ is nonempty as well. 
To prove assertion (2) pick any $x,y\in \partial R$ with $x\nsim y$. From difference constraints (\ref{eq:11satconstraints}), we have $a_x\ne a_y$ in all solutions of $\Phi$, whence $x$ and $y$ belong to twin cells of $\cC$. For assertion (3) let $a$ be a solution of $\Phi$. Similarly as in (2) for all $x,y\in\partial R$ if $a_x=a_y$ then $x$ and $y$ must be adjacent. Thus $\partial R$ can be partitioned into two cliques (of size at most $\omega(G)$) for each truth value. Moreover by one each cell $C\in\cC$ overlaps with $\partial R$ and is disjoint with all other cells. thus $|\cC|\le \partial R\le 2\omega(G)$.
\end{proof}

For two non-empty cells $C,C'\in \cC$, we define an \emph{$a$-arc} $C\rightarrow_A C'$ if $\partial C$ contains a vertex $x$ and $\partial C'$ contains a vertex $y$ adjacent to $x$ such that $x\rightarrow_A y$.
Analogously, we define a \emph{$b$-arc} $C\rightarrow_B C'$  if $\partial C$ contains a vertex $x$ and $\partial C'$ contains a vertex $y$ such that $x\rightarrow_B y$. By Lemma \ref{l:forgotten}, if
$C$ contains a vertex $u$ and $C'$ contains a vertex $v$ such that $u\rightarrow_A v$ (respectively, $u\rightarrow_B v$), then necessarily $C\rightarrow_A C'$ (respectively, $C\rightarrow_B C'$).
The meaning of these arcs between cells is the following. If $C\rightarrow_A C'$, then in all solutions of $\Phi(ab)$ such that $a_{z}=1$ for any $z\in C$ we must also have $a_{z'}=1$ for any $z'\in C'$.
Equivalently, all m-halfspaces of $\Hm(ab)$ %
containing the cell $C$ must also contain $C'$.

From the definition of cells and the equivalence (\Cref{l:halfspaces-from-SAT}) between the solutions of $\Phi(ab)$ and %
m-halfspaces $\Hm(ab)$, %
we obtain the following properties:
\begin{itemize}
\item[(P1)] the halfspaces $H$ and $\cH$ have the form $H=A\cup A^+, \cH=B\cup B^+$, where $A^+$ and $B^+$ are disjoint unions of cells;
\item[(P2)] if  $C\in \cC$ with $C\subseteq H$ then  $\twin(C)\subseteq \cH$; %
\item[(P3)] if  $C\rightarrow_A C'$ and $C$ belongs to $H$, then $C'$ also belongs to $H$. Analogously, if $C\rightarrow_B C'$ and $C$ belongs to $\cH$, then $C'$ belongs to $\cH$;
\item[(P4)] if  $C$ belongs to $H$ and $C'\rightarrow_B C$, then $C'$ belongs to $H$. Analogously, if $C$ belongs to $H$ and $C'\rightarrow_A C$, then  $C'$ belongs to $\cH$.
\end{itemize}

Only property (P4) requires an argument. If $C'$ belongs to $\cH$, then $C'\rightarrow_B C$ implies that $C$ must be included in $\cH$, contrary to the assumption that $C$ belongs to $H$.

On the set $\cC$ of nonempty cells we consider the directed graphs $\oGamma_a$ and $\oGamma_b$: $\oGamma_a$ has the $a$-arcs as arcs and $\oGamma_b$ has the $b$-arcs as arcs.
For a cell $C'$ we denote by $N^-_a(C')$ the set of all in-neighbors of $C'$ in $\oGamma_a$, i.e., of all cells $C$ such that $C\rightarrow_A C'$. The set of all out-neighbors of $C'$ in $\oGamma_a$ is denoted by
$N^+_a(C')$. Set $\deg^-_a(C')=|N^-_a(C')|$ and $\deg^+_a(C')=|N^+_a(C')|$.
We continue with some properties of the directed graph $\oGamma_a$ (similar results hold for $\oGamma_b$):

\begin{lemma} \label{l:in-out-twin} For any cell $C'$, neither of the sets $N^-_a(C'),N^+_a(C')$ contains a cell $C$ and its twin $\twin(C)\ne \varnothing$.
\end{lemma}

\begin{proof}  Pick any cells $C$ and $\twin(C)$, both non-empty. Let also $\alpha$ be any solution of $\Phi(ab)$ and let $(H,\cH)$ be the pair of complementary halfspaces
of $\cH(ab)$ defined by  $\alpha$, where $A\subseteq H, B\subseteq \cH$.  By (P2), $C$ and $\twin(C)$ belong to distinct halfspaces $H$ and $\cH$, say $C\subseteq H$ and $\twin(C)\subseteq \cH$.
First suppose that both $C$ and $\twin(C)$ belong to $N^-_a(C')$. Since $C\rightarrow_A C'$ and $C$ belongs to the halfspace $H$ containing $A$, necessarily $C'\subseteq H$. This implies that all $a_y, y\in C'$ take value 1 in all true assignments $\alpha$  of $\Phi(ab)$, contrary to the assumption that $\Phi(ab)$ does not contain trivial variables.
Now suppose that both $C$ and $\twin(C)$ belong to $N^+_a(C')$. Since $C'\rightarrow_A \twin(C)$ and $\twin(C)$ belongs to the halfspace $\cH$ containing $B$, we conclude that $C'$ cannot belong to the halfspace $H$ containing $A$. This implies that all $a_y, y\in C'$ take value 0 in all true assignments $\alpha$  of $\Phi(ab)$, contrary to the assumption that $\Phi(ab)$ does not contain trivial variables.
\end{proof}

\begin{lemma} \label{l:emptytwin} If $C,C'\in \cC$ and $C\rightarrow_A C'$, then $\twin(C')=\varnothing$.
\end{lemma}

\begin{proof} By the definition of $a$-arcs and Lemma \ref{l:forgotten}, there exist two adjacent vertices $x\in \partial C, y\in\partial C'$  and a vertex $z\in A\setminus \partial_B A$ such
that $z\sim y$ and $z\nsim x$. Suppose by way of contradiction that $\twin(C')\ne \varnothing$. %
First suppose that $y$ is adjacent to a vertex $y'\in \partial\twin(C')$. By Lemma \ref{l:non-adjacent},   all vertices of $\partial C$ are adjacent to all vertices of $\partial\twin(C')$, thus $x\sim y'$. Since $\Phi(ab)$ is satisfiable, we cannot have $\twin(C')\rightarrow_A C'$. Therefore, $y'$ is adjacent to $z$ by \Cref{l:forgotten}. Consequently, $(x,y',z)$ is an induced path of $G$, yielding $x\rightarrow_A y'$ and thus $C\rightarrow_A \twin(C')$. Since we also have $C\rightarrow_A C'$, we get a contradiction with Lemma \ref{l:in-out-twin}.

Now, suppose that $y$ is not adjacent to a vertex $y'\in \partial\twin(C')$. By Lemma \ref{l:non-adjacent} $x$ is adjacent to all vertices of $\partial\twin(C')$ and thus $x\sim y'$.
If $y'\nsim z$, since $y'\nsim y$, $(y',x,y,z)$ is an induced path of $G$. This implies that $y'\rightarrow_A y$ and hence that $\twin(C')\rightarrow_A C'$. Since this is impossible, we must have
$y'\sim z$, yielding $x\rightarrow_A y'$. Consequently,  $C\rightarrow_A \twin(C')$. Since we also have $C\rightarrow_A C'$, this contradicts Lemma \ref{l:in-out-twin}.
\end{proof}

\begin{lemma} \label{l:in-neighbors} If $C,C',C''$ are three different nonempty cells such that $C\rightarrow_A C'$ and $C''\ne \twin(C)$, then either $C''\rightarrow_A C'$ or $C\rightarrow_A C''$. %
\end{lemma}

\begin{proof} By the definition of $a$-arcs and Lemma \ref{l:forgotten}, there exist $x\in \partial C, y\in \partial C'$   and  $z\in A\setminus \partial_B A$ such
that $x\sim y$, $z\sim y$, and $z\nsim x$. By Lemma \ref{l:emptytwin}, $\twin(C')=\varnothing$, hence $y$ is adjacent to all vertices of $C''$. Since $C''\ne \twin(C)$, $x$ is adjacent to all vertices of $C''\cap \partial R$. Pick any $v\in \partial C''$. Then $v\sim x,y$.
If $v\nsim z$, then $v\rightarrow_A y$, yielding $C''\rightarrow_A C'$. If $v\sim z$, then $x\rightarrow_A v$ and thus $C\rightarrow_A C''$.
\end{proof}

\begin{lemma} \label{l:emptytwinsT} If $\cC$ contains a cell $C'$ such that $\deg^-_a(C')\ge 2$, then $\cT=\varnothing$.
\end{lemma}

\begin{proof} Suppose by way of contradiction that $\cT\ne\varnothing$ and let $C\in \cT$. Pick any two cells $C_1,C_2\in N^-_a(C')$. By Lemma \ref{l:emptytwin}, $\twin(C')=\varnothing$, thus $C\ne \twin(C')$. Also we can suppose that $C$ is different from  one of the cells $\twin(C_1)$ or $\twin(C_2)$, say $C\ne \twin(C_1)$. Indeed, if at least one of $\twin(C_1)$,$\twin(C_2)$ is empty, say $\twin(C_1)=\varnothing$, then clearly $C\neq\twin(C_1)$. If both $\twin(C_1)$,$\twin(C_2)$ exist, then again $C$ is different from at least one of the cells, so again without loss of generality  $C\neq\twin(C_1)$.
We assert that $C\in N^-_a(C')$. Suppose not. Then $C\notin N^-_a(C')\cup \{\twin(C_1)\}$ and by
Lemma  \ref{l:in-neighbors}, we conclude that $C_1\rightarrow_A C$. By Lemma \ref{l:emptytwin}, $\twin(C)$ must be empty, contrary to  $C\in \cT$.
Consequently, $C$ must belong to $N^-_a(C')$. Applying the same argument to $\twin(C)\in \cT$, we also conclude that $\twin(C)\in N^-_a(C')$. But this contradicts Lemma \ref{l:in-out-twin}.
This shows that $\cT=\varnothing$, as required.
\end{proof}

The previous lemmas impose strong constrains on the structure of the graphs $\oGamma_a$ and $\oGamma_b$:

\begin{proposition} \label{l:structureoGamma} For graph $\oGamma_a$,  one of the three complementary possibilities holds:
\begin{itemize}
\item[(1)] $\oGamma_a$ does not contain any arc;
\item[(2)] $\oGamma_a$ contains a vertex $C$ with $\deg^-_a(C)\ge 2$. Then $\cT=\varnothing$ and $\cC$ can be partitioned into the sets
$\cC'=N^-_a(C)$ and $\cC''=\cC\setminus N^-_a(C)$ such that for any $C'\in \cC'$ and any $C''\in \cC''$ we have $C'\rightarrow_A C''$;
\item[(3)] $\oGamma_a$ contains at least one arc and $\deg^-_a(C)\le 1$ for any $C\in \cC$. Then either $\cT=\varnothing$ or $\cT$ consists of a single pair $\{ C,\twin(C)\}$ of twins. Furthermore, in both cases,
$C\rightarrow_A C'$ holds for any cell $C'\in\cC\setminus\{C,\twin(C)\}$ and $\twin(C)$ is either empty or is an isolated vertex of $\oGamma_a$. Furthermore, there are no $a$-arcs between
two cells of $N^+_a(C)=\cC\setminus \{ C,\twin(C)\}$.
\end{itemize}
\end{proposition}

\begin{proof} From their definition, the cases (1)-(3) are complementary and cover all possibilities. First consider case (2), i.e., that  $\oGamma_a$ contains a vertex $C$ with $\deg^-(C)\ge 2$. By  Lemma \ref{l:emptytwinsT}, $\cT=\varnothing$. By Lemma \ref{l:in-neighbors}, if we set
$\cC'=N^-_a(C)$ and $\cC''=\cC\setminus N^-_a(C)$, then for any $C'\in \cC'$ and $C''\in \cC''$ we have $C'\rightarrow_A C''$, as asserted.

Now consider case (3) and let  $C\rightarrow_A C'$ be an arc of $\oGamma_a$. Since $\deg^-_a(C')\le 1$,  $N^-_a(C')=\{ C\}$.  By Lemma \ref{l:in-neighbors},  $C\rightarrow_A C''$ for any $C''\notin \{ C,\twin(C)\}$. By Lemma \ref{l:emptytwin}, $\twin(C'')=\varnothing$ for any such $C''$.
Furthermore, by the same lemma, $\twin(C')=\varnothing$. Consequently, either $\cT=\varnothing$ or $\cT=\{ C,\twin(C)\}$. If $\twin(C)\ne \varnothing$, Lemma \ref{l:emptytwin} implies that $\deg^-_a(\twin(C))=0$. If $\deg^+_a(\twin(C))>0$, then there exists  an arc $\twin(C)\rightarrow_A C''$ for some $C''\ne C$. Since we also have $C\rightarrow_A C''$, we obtain a contradiction with Lemma \ref{l:in-out-twin}. Finally, since $\deg^-_a(C'')\le 1$ %
for any cell $C''\in N^+_a(C)$, we cannot have an arc $C'\rightarrow_A C''$ for some $C'\in N^+_a(C)$.
\end{proof}

\subsection{The structure of m-halfspaces} In this subsection, we translate Proposition \ref{l:structureoGamma} into a theorem about the structure of m-halfspaces $H$ of $\Hm(ab)$, which will be used in our learning results. We know that each such halfspaces $H$ consists of the set $A$ from the pair $(A,B)=\ShadowClosure(\{ a\},\{ b\})$ and the \emph{remainder}  $A^+$, which is a union of certain cells of $\cC$.
Our goal is to efficiently characterize all such  reminders.
For this, we encompass the constraints imposed by the $a$-arcs and $b$-arcs into a single directed graph. Namely,  the graph $\oGamma^*_a$ has $\cC$ as the vertex-set and the $a$-arcs and the reverse of $b$-arcs as the arc-set:
in $\oGamma^*_a$ we define an arc $C\rightarrow C'$ between two cells $C,C'$ if and only if either $C\rightarrow_A C'$
or $C'\rightarrow_B C$. The graph $\oGamma^*_b$ is defined in a similar way and is obtained by reversing the arcs of $\oGamma^*_a$.
From the properties (P1)-(P4), we obtain the following useful property:

\begin{lemma} \label{l:arcs} If $C\rightarrow C'$ is an arc of $\oGamma^*_a$, then any halfspace of $\Hm(ab)$ containing the cell $C$ also contains the cell $C'$.
\end{lemma}

The main property of the graph $\oGamma^*_a$ (and $\oGamma_a$ as its subgraph) %
is acyclicity:

\begin{lemma}\label{l:acyclic} The directed graphs $\oGamma^*_a$ and $\oGamma_a$ are acyclic.
\end{lemma}

\begin{proof} Suppose by way of contradiction that $\oGamma^*_a$ contains a directed cycle $(C_1,C_2,\ldots,C_p)$.
By \Cref{l:arcs} all cells $C_i$ are either all contained in $H$ or all in $\cH$ for each halfspace $H\in\Hm(ab)$. This means that the Boolean variables $a_x$ corresponding to the vertices $x$ in the union of these cells are always jointly together true or false, thus they are equivalent variables of $\Phi(ab)$. This contradicts the fact that $C_1,\dots,C_p$ are different cells. Consequently, $\oGamma^*_a$ is acyclic. Since $\oGamma_a$ is a subgraph of $\oGamma^*_a$, $\oGamma_a$ is also acyclic.
\end{proof}

We recall some notions about ordered sets. A \emph{partially ordered set} or \emph{poset} $(P,\le)$ is a set $P$  with a partial order $\le$, i.e., a reflexive, antisymmetric, and transitive binary relation on $P$.
If $x,y\in P$ and $x\le y$ or $y\le x$, then $x$ and $y$ are \emph{comparable}, otherwise they are \emph{incomparable}. An \emph{antichain} is a subset $X\subseteq P$ of pairwise incomparable elements. An element $x$ is \emph{maximal} (\emph{minimal}) if $x\le y$ ($y\le x$) implies $x=y$. We write $x<y$ if $x\le y$ and $x\ne y$.  %
A \emph{lower set} of a poset $(P,\le)$ is a subset $L$ of $P$ such that $y\in L$ and $x\le y$ implies that $x\in L$. A \emph{principal ideal}  of  $x\in P$ is the set ${\mathcal I}_x=\{ y\in P: y\le x\}$. 
A \emph{linear quasiorder} is a set $P$ with a reflexive, transitive, and linear binary relation $\le$. Equivalently, a linear quasiorder can be defined as a poset $(P,\le)$ in which the ground set $P$ can be partitioned into (maximal) antichains $(P_1,\ldots,P_k)$ such that for any $x\in P_i$ and any $y\in P_j$ with $i<j$ we have $x<y$, i.e., we have a total order on the antichains $P_i$.

The acyclic graph $\oGamma^*_a$ defines a partial order $\le$ on $\cC$: we set $C'\le C$ if in $\oGamma^*_a$ there exists a directed path from $C$ to $C'$. 
We call a subset $\cC'$ of $\cC$
\emph{conflict-free} if for any pair $C,\twin(C)\in \cT$, $\cC'$ contains precisely one of the cells of the pair.
We continue with a characterization of reminders of m-halfspaces  of $\Hm(ab)$.

\begin{theorem}\label{t:reminders=ideals} A set of the form $H=A\cup A^+$ is an m-halfspace of $\Hm(ab)$ if and only if $\cC$ contains a conflict-free lower set $\cL$ such that $A^+=\bigcup_{C\in \cL} C$.
\end{theorem}

\begin{proof} 
Let $H\in\Hm(ab)$ and let $H=A\cup A^+$ and $\cH=B\cup B^+$, where  $A^+$ and $B^+$ are  unions of cells of $\cC$, furthermore each cell  belongs either to $A^+$ or to $B^+$.  Set $\cL:=\{ C\in \cC: C\subseteq A^+\}$. Then $B^+$ is the union of the cells of $\cU=\cC\setminus \cL$. Since $(H,\cH)$ corresponds to a solution $\alpha$ of the 2-SAT formula $\Phi(ab)$ with $\alpha(a_x)=1$ if and only if $x\in H$, the set $\cL$ is conflict-free.  To prove that $\cL$ is a lower set of $(\cC,\le)$, pick any $C\in \cL$ and $C'\in \cC$ with $C'\le C$. Then in $\oGamma^*$ there exists a directed path from $C$ to $C'$. Starting from $C$ and applying Lemma \ref{l:emptytwin} to the arcs of this path, we will conclude that all cells on this path, and in particular $C'$, must belong to $H$. Consequently, $C'$ belongs to $\cL$, hence $\cL$ is a lower set.

Conversely, let $\cL$ be a conflict-free lower set of  $(\cC,\le)$. Let $\cU=\cC\setminus \cL$. Let $H=A\cup A^+, \cH=B\cup B^+$, where
 $A^+=\bigcup_{C\in \cL} C$ and $B^+=\bigcup_{C\in \cU} C$. Define the assignment $\alpha$ by setting $\alpha(a_x)=1$ if $x$ belongs to a cell of $\cL$ and $\alpha(a_x)=0$ if $x$ belongs to a cell of $\cU$. To prove that $H$ and $\cH$ are complementary m-halfspaces of $G$, it suffices to show that $\alpha$ is a solution of $\Phi(ab)$. From the definition of cells and the fact that twin cells belong to distinct sets $\cL$ and $\cU$, the equality and difference constraints  (\ref{eq:1satconstraints})   and (\ref{eq:1satconstraints}) are satisfied. Now, pick any implication constraint (\ref{eq:3satconstraints}): it has the form $a_x\rightarrow_A a_y$ or $a_x\rightarrow_B a_y$ for two vertices $x,y\in \partial R$. If $x$ and $y$ belong to the same cell, then $\alpha(a_x)=\alpha(a_y)$
and we are done. Now suppose that $x\in C$ and $y\in C'$. The constraint $a_x\rightarrow_A a_y$ (or $a_x\rightarrow_B a_y$) implies that $C\rightarrow_A C'$ (respectively, $C\rightarrow_B C'$). %

First suppose that we have the constraint  $a_x\rightarrow_A a_y$, which is encoded by the clause $\overline{a_x}\vee a_y$. If this clause is not satisfied by the assignment $\alpha$,
then $\alpha(a_x)=1$ and $\alpha(a_y)=0$. This implies that $C\in \cL$ and $C'\in \cU$. Since $C\rightarrow_A C'$, we have $C'\le C$, contrary to the assumption that $\cL$ is a lower set. Now suppose that we have the constraint
$a_x\rightarrow_B a_y$, which is encoded by the clause $a_x\vee \overline{a_y}$. If this clause is not satisfied by $\alpha$, then $\alpha(a_x)=0$ and $\alpha(a_y)=1$. This implies that $C'\in \cL$ and $C\in \cU$.
Since $a_x\rightarrow_B a_y$, we have the $b$-arc $C\rightarrow_B C'$ and thus in the graph $\oGamma^*_a$ we have the arc $C'\rightarrow C$. This implies that $C\le C'$, contrary to the assumption that $\cL$ is a lower set of $(\cC,\le)$.  This concludes the proof.
\end{proof}
We call the representation $H=A\cup (\bigcup_{C\in \cL} C)$ of each m-halfspace of $\Hm(ab)$ provided by Theorem  \ref{t:reminders=ideals} a \emph{shadow-cell representation}. Now we specify this shadow-cell representation in each of the cases occurring in Proposition \ref{l:structureoGamma}.

\begin{theorem}[Cell decomposition of m-halfspaces]\label{thm:decompositionCases}
Let %
$\Hm(ab)\neq\varnothing$ with corresponding cells $(\cC,\le)$ and non-empty twins $\cT$. Then at least one of the following applies.
    \begin{enumerate}
        \item $(\scC,\le)$ is an antichain, and $H\in\Hm(ab)$ iff $H=A\cup\bigcup_{C\in \cL} C$ for a conflict-free $\cL\subseteq \cC$. %
        \item $\cT=\varnothing$, and there exists a partitioning $(\cC_1,\dots,\cC_p)$ of $\cC$ such that $H\in\Hm(ab)$ iff $H=A\cup\bigcup_{C\in \cL} C$ for a $\cL=\bigcup_{i=1}^{\ell-1}\cC_i\cup \cC_\ell'$ with $\ell\in[p]$ and $\cC'_\ell\subseteq \cC_\ell$
        \item  $\cT=\{ C,\twin(C)\}$, and $H\in\Hm(ab)$ iff $H=A\cup \bigcup_{C'\in \cL} C'\cup C$ for $\cL=(\cC\setminus\cT)\cup\{C\}$ or $\cL=\cC_0'\cup\{\twin(C)\}$ for some $\cC_0'\subseteq (\cC\setminus\cT)$ %
    \end{enumerate}
\end{theorem}
\begin{proof}
By \Cref{l:structureoGamma} $(\scC,\le)$ is an antichain, or $\cT=\varnothing$, or $\cT=\{ C,\twin(C)\}$.
For each one of these cases we prove the ``and'' part of the statement.
\textbf{Case 1:} $(\scC,\le)$ is an antichain. As $(\cC,\le)$ is an antichain, any subset of $\cC$ is a lower set. Thus following \Cref{t:reminders=ideals}, the set $H=A\cup \bigcup_{C\in\cL}C$ is a halfspace if and only if $\cL$ is conflict-free, i.e., contains exactly one cell of each pair of twins.
\textbf{Case 2: $\cT=\varnothing$}. By \Cref{l:quasiorder} below in this case $(\cC,\le)$ can be partitioned into antichains $(\cC_1,\dots,\cC_p)$ such that $C\le C'$ for all $C\in\cC_i$, $C'\in\cC_j$ with $i< j$. Since $\cT=\varnothing$, any lower set $\cL\subseteq \cC$ is conflict-free. Thus, by Theorem \ref{t:reminders=ideals} $A\cup\bigcup_{C\in \cL} C$ is an m-halfspace. Each such lower set has the
following form: for some $\ell\in[p]$ we pick a subset $\cC'_\ell$ of cells of the antichain $\cC_\ell$. To make this set a lower set we have to add all cells belonging to the antichains $\cC_1,\ldots,\cC_{\ell-1}$. Thus $\cL=\bigcup_{i=1}^{\ell-1}\cC_i\cup \cC'_\ell$.
\textbf{Case 3: $\cT=\{ C, \twin(C)\}$}. We can assume that $(\cC,\le)$ is not an antichain, otherwise case 1 applies.
Let $\cC_0=\cC\setminus \cT$ %
By Proposition \ref{l:structureoGamma},
$\deg^-_a(C')\le 1$ and $\deg^-_b(C')\le 1$ for any $C'\in \cC$. First suppose that both graphs $\oGamma_a$ and $\oGamma_b$ have arcs. Applying Proposition \ref{l:structureoGamma}(3) to $\oGamma_a$ and
$\Gamma_b$, we conclude that $C\rightarrow_A C'$ for any $C'\in \cC_0$ and that
$C\rightarrow_B C'$ or $\twin(C)\rightarrow_B C'$ for any $C'\in \cC_0$. If $C\rightarrow_A$
and $C\rightarrow_B C'$, then in $\oGamma^*_a$ we will get a directed cycle $(C,C',C)$ of
length 2, which is impossible by Lemma \ref{l:acyclic}. Consequently, we have
$\twin(C)\rightarrow_B C'$, yielding $C'\rightarrow \twin(C)$ in $\oGamma^*_a$ for any $C'\in \cC_0$. Consequently, in $(\cC,\le)$ we have $\twin(C)\le C'\le C$ for all $C'\in \cC_0$.
Since there are no $a$-arcs or $b$-arcs between two cells of $\cC_0$, $\cC_0$ is an antichain of $(\cC,\le)$. Then the conflict-free
lower sets of $(\cC,\le)$ are exactly the sets of the form $\cC'_0\cup \{\twin(C)\}$,
where $\cC'_0$ is an arbitrary subset of $\cC_0$.
Now suppose that $\oGamma_a$ has arcs but $\oGamma_b$ has no arcs. The $\le$ coincides with the partial order
defined by the graph $\oGamma_a$. From Proposition \ref{l:structureoGamma}(3) it follows that in the poset $(\cC,\le)$ we have $C'\le C$ for any $C'\in \cC_0$, $\cC_0$ is an antichain, and that $\twin(C)$ is an isolated element.
\end{proof}

\begin{lemma} \label{l:quasiorder} If $\cT=\varnothing$, then  $(\cC,\le)$ is a linear quasiorder.
\end{lemma}

\begin{proof} Consider the following  partition $(\cC_1,\ldots,\cC_p)$ of $\cC$ into antichains of  $(\cC,\le)$. Let $\cC_1$ be the set of minimal elements of $(\cC,\le)$. For $2\le i\le p$,  let $\cC_i$ be the set of minimal elements of the poset $(\cC\setminus \bigcup_{j=1,\dots,i-1} \cC_j,\le)$. We assert that $C'\le C''$ for any $C'\in \cC_{i'}, C''\in \cC_{i''}$ with $i'<i''$.
From the definition of the partition $(\cC_1,\ldots,\cC_p)$, the antichain $\cC_{i'}$ contains  a cell $C$ such that $C\le C''$. If $C=C'$, then we are done. So, suppose $C\ne C'$. By the definition of $\le$, in $\oGamma^*_a$ there exists a directed path $R$  from $C''$ to $C$. Let $C_0\rightarrow C$ be the last arc of the path $R$. First suppose that $C_0\rightarrow C$ is the $a$-arc of $\oGamma_a$. Since $\cT=\varnothing$, by Lemma \ref{l:in-neighbors} either $C'\rightarrow_A C$ or $C_0\rightarrow_A C'$. Since $C$ and $C'$ belong to the antichain $\cC_{i'}$, $C'\rightarrow_A C$ is impossible. Therefore $C_0\rightarrow_A C'$ and thus $C'\le C_0$. Since $C_0\le C''$ because $C_0$ is accessible from $C''$ by the subpath of $R$, by transitivity of $\le$, we conclude that $C'\le C''$. Now suppose that  $C_0\rightarrow C$ is the reverse of the $b$-arc
$C\rightarrow_B C_0$. Since $\cT=\varnothing$, by Lemma \ref{l:in-neighbors} for $\oGamma_b$, either $C\rightarrow_B C'$ or $C'\rightarrow_B C_0$. Again $C\rightarrow_B C'$ is impossible because $C,C'$ belongs to the antichain $\cC_{i'}$. But if  $C'\rightarrow_B C_0$, then in $\oGamma^*_a$ we will have $C_0\rightarrow C'$ and thus again $C'\le C_0$. Since $C_0\le C''$, we deduce that $C'\le C''$. This concludes the proof.
\end{proof}
Finally, we show that \Cref{thm:unifiedDecomposition} follows from \Cref{thm:decompositionCases}.
\begin{proof}[of \Cref{thm:unifiedDecomposition}]
    We show that each case of \Cref{thm:decompositionCases} can be represented as required for the cell decomposition of \Cref{thm:unifiedDecomposition}. See also the proof of \Cref{thm:decompositionCases}.
    In Case 1, $p=1$ and $\cC_1=\cC$. All pairs $C_i'$,$C_i''$ are given by each cell and its twin.
    In Case 2, $\cC$ can be partitioned into the antichains $\cC_1,\dots,\cC_p$. The pairs $C_i'$,$C_i''$ are given by each nonempty cell $C_i$ and its empty twin. In Case 3,  we set $p=2$, $\cC_1=\cC_0$, and $\cC_2=\{C,\twin(C)\}$, where $\cC_0=\cC\setminus \cT$. For $C_i\in\cC_2$ we set $C_i'=C$ and $C''_i=\twin(C)$. For all other $C_i\in\cC_0$ we set $C_i'=C_i$ and $C_i''=\varnothing$. We show that this representation yields the exact same halfspaces. First let $H$ be given by $\cC'_0\cup\{\twin(C)\}$ for some $\cC_0'\subseteq\cC_0$. This corresponds to $\ell=1$ and the selection of a cell of each pair $C_i',C_i''$ in $\cC_1$ corresponding to $\cC'_0$. Second let $H$ be given by $\cC_0\cup \{C\}$. This corresponds to $\ell=2$ and the selection of $C\in\cC_2$. Note that for $\ell=2$ and the selection of $\twin(C)\in\cC_2$ instead, we again get the first case with $\cC'_0=\cC_0$.
\end{proof}

\section{The VC-dimension of monophonic halfspaces}\label{sec:VC-dim}
The goal of this section is to provide an efficiently computable characterization of the VC-dimension of $\Hm(G)$ up to an additive constant. For every $ab \in E$ let $d_{ab}=\VC(\Hm(ab))$. %
\begin{theorem}[VC-dimension of m-halfspaces] \label{thm:connected-components-vc-dim}
Every graph $G=(V,E)$ satisfies 
\[\hat d \le \VC(\Hm(G))\le \hat d + 4\]
where
$\hat d = \max_{ab\in E}d_{ab}$ can be computed in polynomial time.
\end{theorem}
We also show that $|\Hm(G)|\leq m2^{d}+2$, see \Cref{sub:enumeration}. %

\subsection{Proof of Theorem~\ref{thm:connected-components-vc-dim}}
We start with a characterization of $d_{ab}$.
\begin{theorem}\label{thm:vc}
    Let $ab$ be an edge such that $\Hm(ab)\neq\varnothing$. Let $(\{C_1',C_1'',\dots,C_k',C_k''\},\scC_1,\dots,\scC_p)$ be the cell decomposition of $\Hm(ab)$ as given by \Cref{thm:unifiedDecomposition}. Then
    \[d_{ab}=\max_{j\in[p]} |\scC_j|\,.\]
\end{theorem}
\begin{proof}
Let $j\in[p]$. For each $C_i\in\scC_j$ we can select a representative point $c_i\in C_i$. Note that exactly one of the two sets $C_i'$,$C_i''$ contains $c_i$. By \Cref{thm:unifiedDecomposition} it follows that the set of all representatives $S=\{c_i: C_i\in\scC_j\}$ is shatterable. Indeed, for every $S'\subseteq S$ there is a selection of cells that contains $S'$ but no representative from $S\setminus S'$. Thus for every $S'$ there is an m-halfspace $H$ such that $S \cap H = S'$. Hence $d_{ab}\ge \max_{j\in[p]} |\scC_j|$. 

For the other direction, let $S$ be a shatterable set. Assume that there exist $c,c'\in S$ that are contained in different $\scC_j$, $\scC_{j'}$; without loss of generality $j<j'$. Moreover let $c\in C_r$ and $c'\in C_t$ for some $r,t\in\mathbb{N}$. First assume $c\in C_r''$. Then by \Cref{thm:unifiedDecomposition} as $j<j'$ any halfspace containing $C_r''$ also contains $C_t''$ and does not contain $C_t'$. This contradicts that $c,c'$ are shatterable, i.e., there are no two halfspaces $H,H'$ with $c\in H,H'$ but $c'$ is only contained in one them. Now let $c\in C_r'$. Any halfspace not containing $C_r'$ must contain $C_r''$ and thus the previous argument applies. Here there are no halfspaces that do not contain $c$ and disagree on $c'$. Thus $S$ is contained in a single $\scC_\ell$. By \Cref{thm:unifiedDecomposition} we see that for any $S'\subseteq S$, we can select the respective cells containing just $S'$ and no points from $S$. This shows that $|\scC_\ell|\ge d_{ab}$, completing the proof.
\end{proof}
Note that, by \cref{thm:unifiedDecomposition}, one can compute $d_{ab}$ and thus $\hat d$ in polynomial time. This proves the last claim of \cref{thm:connected-components-vc-dim}.
It remains to prove the bound on $\hat d$.
The lower bound follows trivially from the definition of $\hat d$ and the fact that $\Hm(ab) \subseteq \Hm(G)$ for every $ab\in E$.
We thus turn to the upper bound. We proceed in several steps.

Let  $\Hmm(ab)=\Hm(ab)\cup \Hm(ba)\cup \{ V,\varnothing\}$. %
It is easy to see that 
\begin{equation}\label{eq:HmplusVC}
\VC(\Hmm(ab))\le d_{ab}+2\,.
\end{equation}

We say that a set $S\subseteq V$ is \emph{almost shattered} by
a set family $\mathcal F$ is for any  $S'\subsetneq S$ with $1<|S'|<|S|-1$, there exists $F\in \mathcal F$ such that $F\cap S=S'$. %
Clearly, shattering implies almost shattering.  %
We continue with the properties of sets almost shattered by $\Hm(G)$.
\begin{lemma} \label{l:quadruplets}  If $S\subseteq V$ with  $|S|\ge 4$ is almost shattered by $\Hm(G)$, then for any quadruplet $\{x,y,z,w\}$ of distinct points of $S$,  $\cm(x,y)\cap \cm(z,w)=\varnothing$.
\end{lemma}
\begin{proof}  Since $S$ is almost shattered by $\Hm(G)$ and $|S|\ge 4$, there exists $H\in \Hm(G)$ such that
$H\cap S=\{ x,y\}$. Therefore, $z,w\in \cH$. Since $H$ and $\cH$ are m-convex, $\cm(x,y)\subseteq H$ and $\cm(z,w)\subseteq \cH$. Since $H$ and $\cH$ are disjoint, we get $\cm(x,y)\cap \cm(z,w)=\varnothing$.
\end{proof}
Next, we will use some properties of sets $S=\{ a_0,a_1,\ldots,a_k\}$ satisfying the property $\cm(x,y)\cap \cm(z,w)=\varnothing$ for all $x,y,z,w\in S$, established by \citet[][Lemmas 6.3 and 6.4]{duchet1988convex}. Let  $P_i$ be a shortest $(a_0,a_i)$-path for
$i\in[k]$. Let $b_1$ be a vertex of $P_1$ closest to $a_1$  having a neighbor in some of the paths $P_2,\ldots,P_k$. By \citet[][Claim]{duchet1988convex}, $b_1$ is adjacent to a vertex from each of the paths $P_2,\ldots,P_k$. On each $P_i, i=2,\ldots,k$ pick a vertex $b_i$ adjacent to $b_1$ and closest to $a_i$.
Then the vertices $b_1,\ldots,b_k$ satisfy the following properties.
\begin{lemma} \label{l:duchet1} \cite[Lemma 6.3 and its proof]{duchet1988convex} The set $\{ b_1,b_2,\ldots,b_k\}$ is a clique of size $k$ of $G$. For any $i,j=1,\ldots,k$ such that $i\ne j$,  $b_i\in \cm(a_i,a_j)$ holds.
\end{lemma}
\begin{lemma} \label{l:almostshatteredclique} If the set $S=\{ a_0,a_1,\ldots, a_k\}$ is almost shattered by $\Hm(G)$, then the set $\widetilde{S}=\{ b_0,b_1,\ldots,b_k\}$, where $b_0=a_0$, is also almost shattered by $\Hm(G)$.
\end{lemma}
\begin{proof} Pick any $S'\subseteq S$ with $1<|S'|<|S|-1$. Let $H\in \Hm(G)$ be such that $H\cap S=S'$. Let $\widetilde{S}'=\{ b_i: a_i\in S'\}$. We assert that $H\cap \widetilde{S}=\widetilde{S}'$.
First, let  $a_0\notin S'$. Pick any $a_i\in S'$.
Since $|S'|>1$, there exists $a_j\in S', j\ne 0,i$. By Lemma~\ref{l:duchet1} and m-convexity of $H$, $b_i\in \cm(a_i,a_j)\subseteq H$. Thus $\widetilde{S}'\subseteq H\cap \widetilde{S}$. If there exists $b_\ell\in H$
such that $a_\ell\notin S'$, then since $a_0,a_\ell\in \cH$ and $b_\ell\in \cm(a_0,a_\ell)\cap H$, a contradiction with m-convexity of $\cH$. Thus $H\cap \widetilde{S}=\widetilde{S}'$.
Now, let $a_0\in S'$.  Pick any $a_i\in S'$ with $i>0$. Since $b_i\in \cm(a_0,a_i)\subseteq H$ we conclude that $\widetilde{S}'\subseteq H\cap \widetilde{S}$. Suppose that there exists $b_\ell\in H$ such that $a_ell\notin S'$. Since $|S'|<|S|-1$, there exists $a_j\notin S'$. Then $a_\ell,a_j\in \cH$ and $b_\ell\in \cm(a_\ell,a_j)\cap H$ (by Lemma~\ref{l:duchet1}), a contradiction. Thus $\Hm(G)$ almost shatters
 $\widetilde{S}$.
\end{proof}
The set $\widetilde{S}$ from Lemma~\ref{l:almostshatteredclique} consists of a clique induced by $\{ b_1,\ldots,b_k\}$ and a vertex $b_0=a_0$. In fact, we can suppose that $\widetilde{S}=\{ b_0,b_1,\ldots,b_k\}$ is a clique. For this we proceed as in the proof of Lemma 6.4 of \cite{duchet1988convex}. Suppose that initially the set $S=\{ a_0,a_1,\ldots,a_k\}$ that is  almost shattered by $\Hm(G)$ was selected to minimize the distance sum
$\sigma(S)=\sum_{a_i,a_j\in S} d(a_i,a_j)$. If $S$ is a clique, then we are done. Otherwise, we replace $S$ by $\widetilde{S}$. By Lemma~\ref{l:almostshatteredclique}, $\widetilde{S}$ is almost shattered by $\Hm(G)$. By Lemma~\ref{l:duchet1}, $b_i\in \cm(b_i,b_j)$ for any $j\ne i$ and $b_i\sim b_j$. This implies $\sigma(\widetilde{S})\le \sigma(S)$. By the choice of $S$,  $b_i=a_i$ for $i=1,\ldots, k$. Therefore, $a_1,\ldots,a_k$. Since $a_0$ is an arbitrary vertex of $S$, we deduce that $S=\{ a_0,a_1,\ldots,k\}$ is a clique of size $k+1$ of $G$. Summarizing, we obtain the following result.
\begin{lemma} \label{l:almostshatteredclique_bis} If $\Hm(G)$ shatters a set $S$, then $\Hm(G)$ almost shatters a clique $K$ of size $|S|$.
\end{lemma}
\begin{lemma} \label{l:almostshatteredclique_bis_repetita} If $\Hm(G)$ almost shatters a clique $K$ of size $k+3$ and $a,b,w$ are three vertices of $K$,  then  $\Hmm(ab)$ shatters the set $\widetilde K=K\setminus \{ b,w\}$.
\end{lemma}
\begin{proof} Pick any  $K'\subseteq \widetilde K$. If $K'=\varnothing$, set  $H=\varnothing$. If $K'=K^+$, set $H=V$. In both cases, $\widetilde K\cap H=K'$. So, let $K'$ be a proper nonempty subset of $\widetilde K$. Since $K=\widetilde K\cup \{ b,w\}$, we have $1\le |K'|<|K|-1$. First, let $|K'|>1$. Since $K'\subseteq K$ and $K$ is almost
shattered by $\Hm(G)$, there exists $H\in \Hm(G)$ such that $H\cap K=K'$. Since $b\notin K'$, we have $b\in \cH$, thus $H\cap \widetilde K=K'$. It remains to show that $H$ can be selected from $\Hm(ab)$. If $a\in K'$, then  $a\in H$. Since $b\in \cH$,  $H\in \Hm(ab)$.
Now, let $a\notin K'$. Then $1<|K'\cup \{ v\}|\le k+1<|K|-1$, and by almost shattering there exists $H\in \Hm(G)$ such that $H\cap K=K'\cup \{ b\}$. But then $a\in \cH$ and $b\in H$, yielding $H\in \Hm(ab)\subseteq \Hmm(ab)$.
Now, let $|K'|=1$. If $K'=\{ a\}$, pick $H\in \Hm(G)$ such that
$H\cap K=\{ a,w\}$ (it exists since $K$ is almost shattered by $\Hm(G)$). But then $b\in \cH$, yielding $H\in \Hm(ab)$. Now, suppose  $K'=\{ x\}$ and $x\ne a$. Pick $H\in \Hm(G)$ such that
$H\cap K=\{ x,b\}$ (it exists since $K$ is almost shattered by $\Hm(G)$ and $x\ne b$). Since $a\in \cH$, we get $H\in \Hm(ab)$. Hence $\Hmm(ab)$ shatters the clique $\widetilde K=K\setminus \{ b,w\}$.
\end{proof}
By Lemma~\ref{l:almostshatteredclique_bis}, $\Hm(G)$ almost shatters a clique $K$ of size $d$. By Lemma~\ref{l:almostshatteredclique_bis_repetita}, $\Hmm(ab)$ shatters the clique $\widetilde K=K\setminus \{ b,w\}$, where
$a,b,w$ are any three distinct vertices of $K$. Consequently, $\vcdim(\Hmm(ab))\ge d-2$. By \Cref{eq:HmplusVC} we have $\vcdim(\Hmm(ab))\le d_{ab}+2$. Putting the two inequalities together, we get $d\le d_{ab}+4\le \hat d+4$. This proves the theorem.

\subsection{Enumerating and bounding the number of monophonic halfspaces}
\label{sub:enumeration}

\begin{proposition} \label{prop:halfspaces_count_bis}  Every graph $G$ satisfies $|\Hm(G)|\leq m2^{d}+2$.
\end{proposition}
\begin{proof} Each halfspace $H\in \Hm(G)\setminus\{V,\varnothing\}$ belongs to at least one of the sets $\Hm(ab)\cup\Hm(ba)$, where $ab$ is an edge of $G$. Consequently, $|\Hm(G)|\le \sum_{ab\in E} |\Hm(ab)\cup\Hm(ba)|+2.$
Notice also that $\vcdim(\Hm(ab))\le \vcdim(\Hm(G))=d$. Therefore, it suffices to bound $|\Hm(ab)\cup\Hm(ba)|$ for all $ab\in E$. For this we use the three cases on the cells $\cC$ provided by \Cref{thm:decompositionCases}. %
First, suppose that $(\cC,\le)$ is an antichain, selecting a cell from each pair of twins (with one potentially empty) yields a halfspace. This means that  $|\Hm(ab)|=2^t$ for $t=|\cC|/2$. By \Cref{thm:vc} these halfspaces shatter a set of size $t$. Together with $\Hm(ba)$, which shatters the same set, both jointly $\Hm(ab)\cup\Hm(ba)$ shatter the same set together with either $a$ or $b$. Thus $t+1\le d$ and hence $|\Hm(ab)\cup\Hm(ba)|=2\cdot 2^t\le 2^d$.    %
Now, suppose that $\cC$ contains a single pair of (non-empty) twin cells and $q$ other cells. From \Cref{thm:decompositionCases} we see that in this case we have $|\Hm(ab)|=2^q+1\le 2^{q+1}$. As before $\Hm(ab)$ shatters a set of size $q$ and $\Hm(ab)\cup\Hm(ba)$ additionally either of $a$ or $b$. Thus $q+1\le d$ and yet again $|\Hm(ab)\cup\Hm(ba)|\le 2^d$.
Finally, suppose that $(\cC,\le)$ is a linear quasiorder $(\cC_1,\ldots,\cC_p)$ with $p$ antichains. Let $q=\max_{i\in[p]}|\cC_i|$. By \Cref{thm:unifiedDecomposition} we have $|\Hm(ab)|=2^q$ and by \Cref{thm:vc} there is a shatterable set of size $q$. As in the first case $\Hm(ab)\cup\Hm(ba)$ additionally shatter $a$ or $b$ and thus $q+1\le d$. Thus again $|\Hm(ab)\cup\Hm(ba)|\le 2^d$.  %
Since $\cT=\varnothing$, \Cref{l:non-adjacent} implies that $\partial R\cup \{ a,b\}$ induces a clique of $G$. This implies that the number of edges of $G$ running between $H\cap \partial R$ and $\cH\cap \partial R$ is at least $|\partial R|$. Consequently, each halfspace  $H$ of $\Hm(ab)$ is counted at least $|\partial R|$ times in the sum $\sum_{ab\in E} |\Hm(ab)\cup\Hm(ba)|$. Since  $p\le |\partial R|$, we conclude that $|\Hm(G)|\le m2^d$.
\end{proof}
Combining \Cref{prop:halfspaces_count_bis} with a folklore listing algorithm yields the following result.
\begin{corollary}\label{prop:enumerationFPT}
    One can enumerate $\Hm(G)$ in time $|\Hm(G)|\poly(n)\le 2^d\poly(n)$.
\end{corollary}
\begin{proof}
    By \Cref{thm:monophonic-halfspace-separation} we can compute consistent hypotheses in polynomial time.
    Thus, we can run a  listing algorithm (e.g., \citealt[][Lemma~3.2.2]{blumer1989learnability}) to enumerate $\Hm(G)$ in time $|\Hm(G)|\poly(n)$. By \Cref{prop:halfspaces_count_bis} this is at most $2^{d}\poly(n)$.
\end{proof}

\section{Labeled sample compression schemes}\label{sec:LSCS}
In this section we describe a compression scheme of size $4\omega(G)$ for $\Hm(G)$. %
A \emph{labeled sample compression scheme} (\emph{LSCS}) \citep{LiWa,FlWa} of size $k$ for a hypothesis space
$\CC \subseteq \{ \pm 1\}^V$ is defined by a \emph{compressor} %
$\alpha: \{ \pm 1,0\}^V \to \{ \pm 1,0\}^V$ and
a \emph{reconstructor} %
$\beta: \{ \pm 1,0\}^V  \to \{ \pm 1\}^V$ such that, for any realizable
sample $X\in \Sam(\CC)$, %
$\alpha(X)\preceq X\preceq \beta(\alpha(X))$ and $|\underline{\alpha}(X)|\le k$. %
Hence, $\alpha(X)$ is a signed vector with support of size $\le k$ such that $\alpha(X)\preceq X$, and
$\beta(\alpha(X))$  is a $(-1,+1)$-sign vector compatible with the whole sample $X$. %
It suffices to define $\alpha$ only on $\Sam(\CC)$ %
and  $\beta$ only on the image of $\alpha$. The condition $X\preceq \beta(\alpha(X))$
is equivalent to $\beta(\alpha(X))|\underline{X}= X$. %
If $\beta(\alpha(X))$ belongs to $\CC$, then $(\alpha,\beta)$ is called a \emph{proper LSCS}. Finally, a LSCS $(\alpha,\beta)$ is called \emph{stable} \citep{BousquetHMZ20} if  for any $X\in \Sam(\CC)$ and any $v \in \underline{X}\setminus \underline{\alpha}(X)$ (any $v$ that is not used for compression) the equality $\beta(\alpha(X'))=\beta(\alpha(X))$ holds, where  $X'$ is obtained from $X$ by setting $X'_v=0$ and $X'_u=X_u$ for any $u\ne v$. Equivalently, a LSCS $(\alpha,\beta)$ is \emph{stable} if for any $X\in \Sam(\CC)$ and any $X'$ such that $\alpha(X)\preceq X'\preceq X$,  $\beta(\alpha(X')=\beta(\alpha(X))$ holds.
For graphs, any preprocessing on the input graph $G$, such as a labeling or an embedding of $G$, is permitted and known to both the compressor and the reconstructor.
Our main result is:
\begin{theorem} [Poly-time proper stable compression of m-halfpaces] \label{t:m-halfspaceCompression-2} For any graph $G=(V,E)$ the  space $(V,\Hm(G))$ admits a proper stable LSCS of size $4\omega(G)$ where the compressor and the decompressor both run in time $\poly(n)$.
\end{theorem}
Unlike many of our other results, \Cref{t:m-halfspaceCompression-2} does not rely on the cell decomposition (\Cref{thm:unifiedDecomposition}), but on the fact that monophonic convexity has Carathéodory number at most 2. 
\cite{FlWa} asked whether
any hypothesis space $\CC$ of VC-dimension $d$ has a LSCS of size~$O(d)$.
This remains one of the oldest open problems in machine learning; our result leaves it open for the special case of $\Hm(G)$, too.

\subsection{A labeled sample compression scheme for monophonic halfspaces}
We use the notion of mutual imprints. Let $A$ and $B$ be two disjoint m-convex sets of $G$. Two vertices $a\in A$ and $b\in B$ are called \emph{mutual imprints} of $A,B$ if there exists an $(a,b)$-path $P$ of $G$ such that $P\cap A=\{a\}$ and $P\cap B=\{ b\}$. Denote by $\Imp_B(A)$ the set of all vertices $a\in A$ such that there exists $b\in B$ so that $a$ and $b$ are mutual imprints;  $\Imp_A(B)$ is defined analogously.
Let $X\in \Sam(\Hm(G))$ and set $A=\cm(X^+),B=\cm(X^-)$. Since $X$ is realizable, $A$ and $B$ are disjoint. Consider the sets
$\Imp_B(A)$ and $\Imp_A(B)$. By Lemma~\ref{mutual-imprints}, $\Imp_B(A)$ and $\Imp_A(B)$ are cliques, thus each of them has size at most $\omega(G)$.  Since the Carathéodory number of monophonic convexity is 2 (Lemma~\ref{mconv-recall}(2)), there exists $X'\subseteq X^+$ of size at most $2\omega(G)$ such that any  $a\in \Imp_B(A)$ belongs to $\cm(x,x')$ for some $x,x'\in X'$. Analogously, there exists $X''\subseteq X^-$
of size at most $2\omega(G)$ such that any  $b\in \Imp_A(B)$ belongs to $\cm(y,y')$ for some $y,y'\in X''$. We set $\alpha^+(X)=X'$ and $\alpha^-(X)=X''$.
To obtain a stable LSCS,  consider an arbitrary  fixed total order $<$ of vertices of $G$. Suppose that all pairs of vertices  of $G$ are ordered lexicographically with respect to $<$. To build $X'$, we consider the vertices of $\Imp_B(A)$ in this order. For a current  $a\in \Imp_B(A)$ we consider the pairs  of  $X^+$ in  lexicographic order and we insert in $X'$ the first pair $\{ x',x''\}$ such that $a\in \cm(x',x'')$. Since the Carathéodory number of m-convexity is 2, such a pair necessarily exists. The set $X''$ is constructed in similar way using the pairs of $X^-$.

Now, suppose that $Y$ is a sign vector  from the image of the map $\alpha$. Then $|Y^+|\le 2\omega(G)$ and $|Y^-|\le 2\omega(G)$. The reconstructor builds the m-convex hulls $A^+:=\cm(Y^+)$ and $B^-=\cm(Y^-)$ in the graph $G$ and returns as $\beta(Y)$ the halfspace $H$ from any pair $(H,\cH)$ of complementary m-halfspaces separating $\Imp_{B^-}(A^+)$ and $\Imp_{A^+}(B^-)$, i.e., $\Imp_{B^-}(A^+)\subseteq  H$ and $\Imp_{A^+}(B^-)\subseteq \cH$.
Our main technical result, which immediately implies \Cref{t:m-halfspaceCompression-2}, is:
\begin{theorem}\label{t:m-halfspaceCompression} The pair $(\alpha,\beta)$ is a proper stable LSCS of size $4\omega(G)$ for  $\Hm(G)$. %
The maps $\alpha(\cdot)$ and $\beta(\alpha(\cdot))$ can be computed in polynomial time.
\end{theorem}
The pseudocode of the algorithms for compression and reconstruction are as follows.
\begin{algorithm} \label{algorithm-compression}
\DontPrintSemicolon
\KwInput{a graph $G=(V,E)$ and a realizable sample $X\in \Sam(\Hm(G))$}
\KwOutput{the compression $\alpha(X)$ of $X$}
$A=\cm(X^+)$, $B=\cm(X^-)$\;
compute the mutual imprints $\Imp_B(A)$ and $\Imp_A(B)$\;
compute the sets  $X'\subseteq X^+$ and $X''\subseteq X^-$ such that
\[ \Imp_B(A)\subseteq \bigcup_{x,x'\in X'} \cm(x,x') \mbox{ and } \Imp_A(B)\subseteq \bigcup_{x,x'\in X''} \cm(x,x')\]
\textbf{return} $\alpha(X)$, where $\alpha^+(X)=X'$ and $\alpha^-(X)=X''$\;
\caption{Compression$(X)$}
\end{algorithm}
\begin{algorithm} \label{algorithm-reconstruction}
\DontPrintSemicolon
\KwInput{a graph $G=(V,E)$ and sign vector $Y$ from the image of the map $\alpha$ ($Y=\alpha(X)$ for $X\in \Sam(\Hm(G))$}
\KwOutput{an m-halfspaces $H$ compatible with the sample $X$}
 $A^+=\cm(Y^+)$, $B^-=\cm(Y^-)$\;
 $(H,\cH)=${\sc HalfspaceSeparation}$(A^+,B^-)$\;
\textbf{return}  the halfspace $H$ as $\beta(Y)$\;
\caption{Reconstruction$(Y)$}
\end{algorithm}
\begin{lemma} \label{mutual-imprints} If $A,B$ are disjoint m-convex sets, then $\Imp_B(A),\Imp_A(B)$ are nonempty cliques.
\end{lemma}
\begin{proof} If $a\in A$ and $b\in B$ are such that $d(a,b)=d(A,B)$, then any shortest $(a,b)$-path $P$ intersects the sets $A$ and $B$ in $a$ and $b$, respectively. Thus the sets $\Imp_B(A)$ and $\Imp_A(B)$ are nonempty. Now, pick any  $a,a'\in \Imp_B(A)$ and suppose that $a\nsim a'$. Pick $b,b'\in \Imp_A(B)$ such that $\{ a,b\}$ and $\{ a',b'\}$ are mutual imprints (it may happen that $b=b'$). Let $P$ be an $(a,b)$-path and $P'$ be an $(a',b')$-path such that $P\cap A=\{ a\}, P\cap B=\{ b\}$ and $P'\cap A=\{ a'\}, P'\cap B=\{ b'\}$ ($P$ and $P'$ are not necessarily disjoint). Let also $R$ be a path of $B$ between the vertices $b$ and $b'$. Then $W=P\cup R\cup P'$ is an $(a,a')$-walk. All vertices of $W$ except $a$ and $a'$ do not belong to $A$. From $W$ we select a simple $(a,a')$-path $W'$ and from $W'$ one select an induced $(a,b)$-path $Q$. Since $a\nsim a'$ and all intermediate vertices of $Q$ do not belong to $A$, we get a contradiction with the assumption that $A$ is m-convex.
\end{proof}
\begin{lemma} \label{mutual-imprints-sep} If $A$ and $B$ are  disjoint m-convex sets  and $a'\in A$ and $b'\in B$ are not mutual imprints, then any $(a',b')$-path  $P$ contains vertices $a\in \Imp_B(A)$ and  $b\in \Imp_A(B)$ such that $a$ and $b$ are mutual imprints of $A,B$.
\end{lemma}
\begin{proof} Move from $a'$ to $b'$ along  $P$. Denote by $a$ the last vertex of $A\cap P$ and $b$ be the first vertex of $P\cap B$ that we will meet after $a$. From the choice of $a$ and $b$, the subpath of $P$ from $a$ to $b$ intersects  $A$ only in $a$ and  $B$ only in $B$. Thus $a$ and $b$ are mutual imprints.
\end{proof}
\begin{lemma} \label{mutual-imprints-bis} If $A,B$ and $A',B'$ are two pairs of disjoint m-convex sets such that $\Imp_B(A)\subseteq A'\subseteq A$ and $\Imp_A(B)\subseteq  B'\subseteq B$, then $\Imp_{B'}(A')=\Imp_B(A)$ and $\Imp_{A'}(B')=\Imp_A(B)$.
\end{lemma}
\begin{proof} First we show that $\Imp_B(A)\subseteq \Imp_{B'}(A')$ and $\Imp_A(B)\subseteq \Imp_{A'}(B')$. Pick any pair $a\in \Imp_B(A)$ and $b\in \Imp_A(B)$ of mutual imprints of $A$ and $B$. Then there exists a path $P$ between $a$ and $b$ such that $P\cap A=\{ a\}$ and $P\cap B=\{ b\}$. Since $a\in A'\subseteq A$ and $b\in B'\subseteq B$, we also have $P\cap A'=\{ a\}$ and $P\cap B'=\{ b\}$, thus $a$ and $b$ are mutual imprints of $A'$ and $B'$.
To prove that $\Imp_{B'}(A')\subseteq \Imp_{B}(A)$ and $\Imp_{A'}(B')\subseteq \Imp_{A}(B)$, pick an arbitrary pair $a'\in \Imp_{B'}(A')$ and $b'\in \Imp_{A'}(B')$ that are mutual imprints of $A',B'$. Let $P$ be a path such that $A'\cap P=\{ a'\}$ and $B'\cap P=\{ b'\}$. If $a'$ and $b'$ are not mutual imprints of $A,B$, by Lemma~\ref{mutual-imprints-sep}, $P$ contains two vertices $a\in \Imp_B(A)$ and $b\in \Imp_A(B)$ that are mutual imprints of $A,B$. Since $a\in \Imp_B(A)\subseteq A'$ and $b\in \Imp_B(A)\subseteq B'$, we get a contradiction with  $A'\cap P=\{ a'\}$ and $B'\cap P=\{ b'\}$.
\end{proof}
\paragraph{Proof of Theorem~\ref{t:m-halfspaceCompression}.}
 Let $X\in \Sam(\Hm(G))$, $Y=\alpha(X)$ and $H=\beta(Y)$. Let also $A=\cm(X^+), B=\cm(X^-)$ and $A^+:=\cm(Y^+),B^-=\cm(Y^-)$. Since $X$ is a realizable sample of $\Hm(G)$, $A\cap B=\varnothing$. Since $Y^+=\alpha^+(X)\subseteq X^+\subseteq A$ and $Y^-=\alpha^-(X)\subseteq X^-\subseteq B$, we conclude that $A^+\cap B^-=\cm(Y^+)\cap\cm(Y^-)\subseteq A\cap B=\varnothing$. This shows that  $\Imp_{B^-}(A^+)$ and $\Imp_{A^+}(B^-)$ are disjoint, whence the halfspace $H=\beta(\alpha(X))$ is well-defined. By the definition of $\alpha$, $\Imp_B(A)\subseteq \cm(\alpha^+(X))=A^+$ and $\Imp_A(B)\subseteq \cm(\alpha^-(X))=B^-$. From the definition of mutual imprints and $A^+\subseteq A, B^-\subseteq B$, we get $\Imp_B(A)\subseteq \Imp_{B^-}(A^+)$ and $\Imp_A(B)\subseteq \Imp_{A^+}(B^-)$.

We assert that the pair of m-halfspaces $H,\cH$ separates $A=\cm(X^+),B=\cm(X^-)$. Pick any $x\in A$ and $b\in B$. Let $Q$ be an arbitrary induced $(x,y)$-path. Let $a$ be the furthest from $x$ (along $Q$) vertex of $A$ and let $b$ be the furthest from $y$ (along $Q$) vertex of $B$. Let $P$ be the subpath of $Q$ between the vertices $a$ and $b$. From the choice of  $a$ and $b$,  $P\cap A=\{ a\}$ and $P\cap B=\{ b\}$, whence $a,b$ are mutual imprints, i.e., $a\in \Imp_B(A)\subseteq \Imp_{B^-}(A^+)\subseteq H'$ and $b\in \Imp_A(B)\subseteq \Imp_{A^+}(B^-)\subseteq H''$. Since $Q$ is an induced path of $G$, any its subpath is also an induced path of $G$. This implies that $a\in \cm(x,b)\subseteq \cm(x,y)$ and $b\in \cm(a,y)\subseteq \cm(x,y)$. Since $a\in H$ and $b\in \cH$, from these inclusions we also deduce that $x\in H$ and $y\in \cH$.
 Consequently, $X^+\subseteq H$ and $X^-\cap H=\varnothing$, thus the halfspace $H=\beta(\alpha(X))$ is compatible with the sample $X$. Consequently, $(\alpha, \beta)$ is a proper LSCS of size at most $4\omega(G).$

 Now we prove that the LSCS $(\alpha,\beta)$ is stable. Again, pick any realizable sample $X\in \Sam(\CC)$ and any $v$ from the support of $X$ such that $v$ is not used for compression ($v$ is not in the support of $\alpha(X)$, i.e., $v\notin X'\cup X''$). Suppose without loss of generality that $v\in X^+$ (the case $v\in X^-$ is analogous). Let $Z$ be the realizable sample such that $Z_v=0$ and $Z_u=X_u$ for any $u\ne v$.   Let $A'=\cm(Z^+)$ and $B'=\cm(Z^-)$. Clearly, $Z^+=X^+\setminus \{ v\}$ and $Z^-=X^-$, whence $A'\subseteq A$ and $B'\subseteq B$.  Since $v\notin X'$, $X'\subseteq Z^+$ and thus $\Imp_B(A)\subseteq \cm(X')\subseteq A'$.
 Since $Z^-=X^-$, we also get  $\Imp_A(B)\subseteq \cm(X'')\subseteq B'$. Therefore, we can use Lemma~\ref{mutual-imprints-bis} for the pairs of disjoint m-convex sets $A,B$ and $A',B'$. By this lemma,  $\Imp_{B'}(A')=\Imp_B(A)$ and $\Imp_{A'}(B')=\Imp_A(B)$.
 To prove that $\beta(\alpha(Z))=\beta(\alpha(Z))$, we will show a stronger equality $\alpha(Z)=\alpha(X)$. Since  $\alpha^-(Z)$ is defined using only the pairs of $Z^-$ and $Z^-=X^-$, we have $\alpha^-(Z)=\alpha^-(X)=X''$. To construct $\alpha^+(Z)$, the vertices of $\Imp_{B'}(A')$ are proceeded in the same order as the vertices of $\Imp_B(A)=\Imp_{B'}(A')$. For a current vertex $a$ from both imprints, the compressor included in $X'=\alpha^+(X)$ the lexicographically smallest pair $\{ x',x''\}$ such that $a\in \cm(x',x'')$. But we know that $X'\subseteq Z^+\subseteq X^+$. Therefore, the compressor will include in $Z'=\alpha^+(Z)$ the same pair $\{ x',x''\}$ for $a$. This establishes that  $\alpha^+(Z)=\alpha^+(X)$ and thus $\alpha(Z)=\alpha(X)$. Consequently, $(\alpha,\beta)$ is a stable LSCS.

 It remains to show that $\alpha$ and $\beta$ can be defined in polynomial time. Let $X\in \Sam(\Hm(G))$. First we construct the m-convex hulls $A=\cm(X^+)$ and $B=\cm(X^-)$. Then we construct the sets
$\Imp_B(A)$ and $\Imp_A(B)$ of mutual imprints.  For this, we pick any $a$ in $A$ and $b$ in $B$, remove from $G$ the sets $A\setminus \{ a\}$ and $B\setminus \{ b\}$ and in the resulting graph $G'$ we check if $a$ and $b$ can be connected by a path. If ``yes'', then we include $a$ in $\Imp_B(A)$ and $b$ in $\Imp_A(B)$, otherwise we pass to the next pair $(a,b)$. With the sets $\Imp_B(A)$ and $\Imp_A(B)$ at hand, we can construct the sets $\alpha^+(X)=X'$ and $\alpha^-(X)=X''$ as follows. Initially set $X'=\varnothing$, then  pick any $a\in \Imp_B(A)$ and any $a',a''\in A$ and test if $a$ belongs to $\cm(a',a'')$; if ``yes'', then include $a',a''$ in $X'$ and pass to the next vertex $a$ of $\Imp_B(A)$. The set $X''$ is defined analogously. Clearly, this takes polynomial time. Finally, to define $\beta(Y)$ for a sign vector $Y$ from the image of $\alpha$, the reconstructor builds the m-convex hulls $A^+:=\cm(Y^+)$ and $B^-=\cm(Y^-)$  and returns as $\beta(Y)$ the halfspace $H$ from any pair of complementary m-halfspaces $H,\cH$ separating $\Imp_{B^-}(A^+)$ and $\Imp_{A^+}(B^-)$. Since $Y\in \Sam(\Hm(G))$, the pair $\{ H,\cH\}$ exists and can be computed in polynomial time by  applying {\sc HalfspaceSeparation}$(\Imp_{B^-}(A^+),\Imp_{A^+}(B^-),\cM)$ (Theorem~\ref{thm:monophonic-halfspace-separation}). This concludes the proof of Theorem~\ref{t:m-halfspaceCompression}.

\section{Supervised learning}\label{sec:supervised} In this section we consider the problem of PAC learning of $\Hm(G)$ in the realizable and agnostic settings. %
Formally, our hypothesis space is $(V,\CC)$ where  $\CC=\Hm(G) \subseteq 2^V$ is defined implicitly by  $G$.
Recall that we denote by $X\in \{ \pm1,0\}^V$ the input sample.
Note that, while standard PAC learning formulations typically measure the running time as a function of the sample size (see, e.g., \citealt{valiant1984theory,shalev2014understanding}), we measure running time as a function of $n=|V|$.

\paragraph{Realizable PAC learning.}
The \emph{realizable} setting assumes $X\in\Sam(\Hm(G))$. 
In this case, our results yield two PAC learning stategies.
First, by \Cref{thm:monophonic-halfspace-separation} we can compute a consistent m-halfspace for any realizable sample $X$ in time $\poly(n)$. Any learner that outputs such a  hypothesis can PAC learn using
$\scO\left(\frac{d\log(1/\varepsilon)+\log(1/\delta)}{\varepsilon}\right)$
samples by standard bounds \citep{blumer1989learnability}. The bound is near-optimal up to the $\log(1/\varepsilon)$ factor; standard lower bounds have the form $\Omega\Bigl(\frac{d+\log(1/\delta)}{\varepsilon}\Bigr)$ \citep{ehrenfeucht1989general}.
Second, we can use the sample compression scheme of \Cref{t:m-halfspaceCompression}.
The scheme is stable, has size $4\omega(G)$, and can be implemented in polynomial time. %
By \citet[][Theorem 15]{BousquetHMZ20} this implies PAC learning with $\scO\left(\frac{\omega(G)+\log(1/\delta)}{\eps}\right)$ samples. Summarizing, we get the following theorem.
\begin{theorem}[Poly-time realizable PAC]\label{thm:polytime_ralizable_PAC}
In the realizable setting, the space $(V,\Hm(G))$ is PAC learnable in polynomial time using $\scO\Bigl(\frac{\min\{\omega(G),\,d\log(1/\eps)\}+\log(1/\delta)}{\varepsilon}\Bigr)$ samples.
\end{theorem}

This makes m-halfspaces one of the few known hypothesis spaces that can be PAC-learned  in polynomial time with the optimal rate (save the gap of $\omega(G)$ and $d$).
Standard examples are intersection-closed spaces \citep{darnstadt2015optimal}, maximum classes, and halfspaces in Euclidean spaces, where the optimal rate was a long-standing open question \citep{BousquetHMZ20}.

\paragraph{Agnostic PAC learning.}
In the \emph{agnostic setting}, the sample $X\in \{ \pm1,0\}^V$  is not necessarily realizable by $\Hm(G)$. %
Our goal here is to efficiently perform ERM (\emph{empirical risk minimization}), i.e., to find an m-halfspace that minimizes the empirical risk
 $L(X,H) = \frac1{|\underline{X}|} \sum_{v\in \underline{X}} \mathbbm{1}\{X(v) \ne \mathbbm{1}_{v\in H}\}$ over all $H\in \Hm(G)$. %
\begin{theorem}[Poly-time ERM]\label{thm:erm-poly}
    There exists an algorithm that, given a graph $G=(V,E)$  and a sample $X\in  \{ \pm1,0\}^V$, computes $\argmin_{H\in\Hm(G)}L(X,H)$  in polynomial time.
\end{theorem}
\begin{proof}
We start with an ERM for $\Hm(ab)$ for each fixed edge $ab$. By \Cref{thm:unifiedDecomposition} there exists is a cell decomposition 
 $(\{C_1',C_1'',\dots,C_k',C_k''\}, \scC_1,\dots,\scC_p)$ such that 
 $H\in\Hm(ab)$ if and only if $H=A^* \cup(\bigcup\limits_{C_i\in\scC_{< \ell}\,\cup\,\scC_\ell'} C'_i)\cup  (\bigcup\limits_{C_i\in(\scC_\ell\setminus\scC_\ell')\,\cup\,\scC_{> \ell}}C''_i)$ 
    for some $\ell\in[p]$ and $\scC_\ell'\subseteq \scC_\ell$. 
    Let $\ell\in[p]$ and denote by $\Hml(ab)$ all halfspaces that can be written in this union by selecting this particular $\ell$. Denote $C_i=C_i'\cup C_i''$ for all $i\in[k]$ and note that the halfspaces in $\Hml(ab)$ only differ on the sets $C_i\in\scC_\ell$. More precisely, any selection of $C_i'$ or $C_i''$ (with $C_i\in\scC_\ell$) gives two halfspaces in $\Hml(ab)$ and the empirical risk of these halfspaces only differs exactly on these sets. This means that for each particular $C_i\in\scC_\ell$ we can simply select the half $C_i'$ or $C_i''$ that minimizes the empirical risk on $C_i$, i.e., select the cell $D_i\in\{C_i',C_i''\}$ with smaller $|\{v\in C_i: X(v)=0\}\cap D_i|+|\{v\in C_i: X(v)=1\}\cap (C_i\setminus D_i)|$. Individually selecting all such $D_i$ thus yields an ERM for $\Hml(ab)$. Note that this is possible in polynomial time as by \Cref{thm:unifiedDecomposition} computing the cell decomposition takes polynomial time. As any halfspaces in $\Hm(ab)$ is in $\Hml(ab)$ for some $\ell$ we get an ERM for $\Hm(ab)$ by taking the halfspace with overall smallest empirical risk from over all classes $\Hml(ab)$ with $\ell\in[p]$.

    To compute an ERM for $\Hm(G)$,
    we first compute the ERM of $\Hm(ab)$ for each edge $ab$ (and $ba$) and take the halfspace $\Hat H$ with empirical risk smallest among all these. As any halfspaces in $\Hm(G)$ is either trivial (i.e., $V$ or $\varnothing$) or in $\Hm(ab)$ or $\Hm(ba)$ for some edge $ab$, we can thus return any hypothesis from $\{V,\varnothing,\hat H\}$ with smallest empirical risk as an ERM. 
\end{proof}

\begin{corollary}[Poly-time agnostic PAC]\label{thm:PAC_agnostic}
In the agnostic setting, the space $(V,\Hm(G))$ is PAC learnable in polynomial time using $\Theta\Bigl(\frac{d+\log(1/\delta)}{\varepsilon^2}\Bigr)$ samples.
\end{corollary}
\begin{proof}
    We compute an ERM by Theorem~\ref{thm:erm-poly}, which yields the optimal agnostic sample complexity $\Theta\Bigl(\frac{d+\log(1/\delta)}{\varepsilon^2}\Bigr)$ \citep{vapnik1974theory, talagrand1994sharper}.
\end{proof}

\section{Active learning of monophonic halfspcaes}\label{sec:active} In this section, we consider active learning of m-halfspaces. Here the algorithm is given a graph $G$ and nature selects a
hypothesis $H\in \Hm(G)$.
The algorithm can query any  $x\in V$ to learn if $x$ belongs or not to $H$. The goal is to output $H$ by making as few queries as possible. This problem is a special case of realizable transductive active learning on a graph \citep{afshani2007complexity, guillory2009label,cesa2010active,dasarathy2015s2} and %
is a variant of query learning \citep{angluin1988queries, hegedHus1995generalized}. %
For a hypothesis space $(V,\scH)$, its query complexity   is the maximum number of queries an optimal algorithm makes over $H \in \scH$. More precisely, for any algorithm $A$ and any $H\in \scH$, let $\qc(A,H)$ be the number of queries $A$ make on $G$ when $H$ is chosen. The query complexity of $A$ on $\scH$ is $\qc(A,\scH)=\max_{H\in\scH} \qc(A,H)$. The \emph{query complexity} of $\scH$ is $\qc(\scH)=\min_A \qc(A,\scH)$. %

\begin{theorem}[Poly-time active learning]\label{thm:active_upper_bound}
  For any graph $G$, %
  \[\qc(\Hm(G))=\scO\left(\hull(G)+\log\diam(G)+\log\omega(G)+d\right)\,.\] %
\end{theorem}

\begin{proof} %
By Lemma~\ref{mconv-recall}(5) we can compute in polynomial time a minimum hull set $X$, i.e.,  a minimum set $X$ with the property $\cm(X)=V$.
Then our algorithm queries all vertices of $X$. If they have all the same label, all vertices of $G$ have the same label. %
Otherwise, we pick  $u',v'$ in $X$ with $u'$ positive and $v'$ negative, and find a shortest $(u',v')$-path $P$. %
As the labels are given by a m-halfspace $H$, $P$ has exactly one edge $e=uv$ such that all vertices on the subpath of $P$ between $u'$ and $u$ are positive and all vertices on the subpath between $v$ and $v'$ are negative. We identify $e$ using $\scO(\log\diam(G))$ queries through binary search on $P$. From this we know that $H\in\Hm(uv)$. Then we use \Cref{thm:unifiedDecomposition} to compute a cell decomposition $(\{C_1',C_1'',\dots,C_k',C_k''\},\scC_1,\dots,\scC_p)$ of $\Hm(uv)$ in polynomial time. Let $C_i=C_i'\cup C_i''$ and select a representative $c_i\in C_i$ for all $i\in[k]$. By \Cref{thm:unifiedDecomposition} (see also \Cref{fig:decomposition}) we have to identify $\ell\in[p]$ and the correct splits into $C_i',C_i''$ of all $C_i\in\scC_\ell$ to fully determine the target halfspace $H$. In particular, $H$ contains all $C_i'\in\scC_{<\ell}$ and all $C_i''\scC_{>\ell}$. We perform the following binary search on the $\scC_1,\dots,\scC_p$: fix a $\scC_j$ and query any representative $c_i$ in a cell $C_i\in\scC_j$. If $c_i\in C_i'$ and its label is positive (or $c_i\in C_i''$ and its label is negative) we know that $\ell\ge j$. If not we know that $\ell\le j$. That way after $\scO(\log p)\le \scO(\log \omega(G))$ queries we can identify the correct $\ell$ in $\{j,j+1\}$, where the inequality follows by \Cref{l:non-adjacent}. In both $\scC_j$ and $\scC_{j+1}$ we query all representatives. This allows us to determine all cells and hence $H$ as discussed. As $|\scC_j|,|\scC_{j+1}|\le d$ by \Cref{thm:vc} the query complexity follows.
\end{proof}

Theorem~\ref{thm:active_upper_bound} should be contrasted with the active learning algorithm of \citet{thiessen2021active} for \emph{geodesic} halfspaces; that algorithm is not polynomial time, %
as it requires solving the minimum geodesic hull set problem, which is APX-hard. Also note that the simpler approach by \citet{BrEsTh24} results in the worse query complexity $\scO\left(\hull(G)+\log\diam(G)+\omega(G)\right)$.
\paragraph{Lower bounds.}
Along the previously mentioned separation axioms, we achieve increasingly tighter lower bounds on the query complexity, eventually matching our upper bound from Theorem~\ref{thm:active_upper_bound} for all $S_3$ graphs.%
\begin{proposition}\label{thm:lower_bounds}
The following holds for the query complexity $\qc(\Hm(G))$:
\begin{itemize}[itemsep=2pt,parsep=0pt,topsep=4pt]
    \item $\qc(\Hm(G))=\Omega(d)$,
	\item if $G$ is $S_2$, then $\qc(\Hm(G))= \Omega(d+\log\diam(G))$, and
	\item if $G$ is $S_3$, then  $\qc(\Hm(G))=\Omega(d+\log\diam(G)+\log\omega(G)+\hull(G))$.
\end{itemize}
\end{proposition}
\begin{proof}
It holds that $\qc(\Hm(G))\ge \log |\Hm(G))|$  \citep{hegedHus1995generalized}. As $|\Hm(G)|\ge 2^d$ the first statement follows.
For $S_2$ graphs we can take take a shortest path $P$ of maximum length. By  $S_2$, for any each $uv$ of $P$ there must be a m-halfspace separating $uv$. Thus $|\Hm(G)|\ge \diam(G)$.
For $S_3$ graphs, let $X$ be a hull set of size $\hull(G)$. Now let $A$ be any active learning algorithm and let $A$ query its first $|X|-1$ nodes $S$, where the oracle always returns positive labels. As $X$ is minimum and $|S|\le k-1$, there is exists a node $x\in V\setminus \cm(S)$. By $S_3$, there exists a halfspace $H_x$ separating $x$ from $S$. The algorithm $A$ cannot distinguish between $V$ and $H_x$ and hence has to continue querying. This implies that $\qc(\Hm(G))\ge \hull_m(G)$. Moreover, on $S_3$ graphs, for any clique $C$ and a node $c\in C$ there is halfspace separating $c$ from $C\setminus\{c\}$. Thus $|\Hm(G)|\ge \omega(G)$ and hence $\qc(G)\ge \log\omega(G)$.
\end{proof}
\section{Teaching monophonic halfspaces}\label{sec:teaching}
In this section, we bound the (recursive) teaching dimension of $\Hm(G)$ by $2d+2$.
In machine teaching, given a hypothesis space $\CC\subseteq \{ \pm 1\}^V$, a teacher presents to a learner a set $T(C)$
of correctly labeled examples from a hypothesis $C\in \CC$ in such a way that the learner can reconstruct $C$ from $T(C)$.  %
A \emph{teaching set} for $C\in \CC$ is a set $T(C)$ of labeled examples such that $C$ is the
unique hypothesis of $\CC$ that is consistent with $T(C)$. Denote by $\cTS(C,\CC)$ the collection of all teaching sets for $C$ and let $\TS(C,\CC)$ be the size of the smallest set of $\cTS(C,\CC)$.
The quantity $\TD(\CC)=\max \{ \TS(C,\CC): C\in \CC\}$ is called the
\emph{teaching dimension}  of $\CC$ \citep{GK95}. A \emph{teaching plan} is a sequence $P=\{  (C_1,T(C_1)),\ldots (C_N,T(C_N))\}$ such that $\CC=\{ C_1,\ldots,C_N\}$ and  $T(C_i)\in \TS(C_i, \{ C_i,\ldots,C_N\})$, for any  $i=1,\ldots,N$ \citep{ZillesLHZ11}. Let $\ord(P)=\max_{\{i=1,\ldots,N\}}|T(C_i)|$. The \emph{recursive teaching dimension} $\RTD(\CC)$ of $\CC$ is the minimum of $\ord(P)$ taken over all teaching plans $P$ for $\CC$ \citep{ZillesLHZ11}.
Let $\wHm(G)=\Hm(G)\setminus \{ \varnothing, V\}$. \Cref{thm:unifiedDecomposition} provides us with an easy way to bound $\TD(\wHm(G))$ and $\RTD(\Hm(G))$ in terms of $d=\vcdim(\Hm(G))$.
\begin{theorem}[Teaching dimension of m-halfspaces] \label{thm:teaching} For a graph $G$, $\TD(\wHm(G))\le 2d+2$ and
$\RTD(\Hm(G))\le 2d+2$. %
\end{theorem}
\begin{proof} First we prove $\TD(\wHm(G))\le 2d+2$.  Pick any $H\in \wHm(G)$. Let $ab$ be an edge with $H\in \Hm(ab)$ and let $(\{C_1',C_1'',\dots,C_k',C_k''\},\scC_1,\dots,\scC_p)$ be the cell decomposition of $\Hm(ab)$ given by \Cref{thm:unifiedDecomposition}. we add $a$ and $b$ with their corresponding labels to the teaching set. Let $\scC_\ell'\subseteq\scC_\ell$ be the sets of cells as used in the cell decomposition of $H$. If $\varnothing\neq\scC_\ell'\subsetneq\scC_\ell$ then $\ell$ is unique. Otherwise choose $\ell$ as the largest index $j$ where $\scC_j'=\scC_j$. For each $C_i\in\scC_\ell$ there are two cells $C_i=C_i'\cup C_i''$. %
Pick any $c_i\in C_i$ as a representative.%
If $c\in C_i'$ and $C_i'\subseteq H$ (or $c\in C_i''$ and $C_i''\in \cH$) then we use $c_i$ with a positive label in the teaching set. Otherwise we use $c_i$ with a negative label. By \Cref{thm:unifiedDecomposition} the labels of the representatives fully determine the labels of $\scC_\ell$. If $\varnothing\neq\scC_\ell'\subsetneq\scC_\ell$, the edge $ab$ and the representatives additionally determine $\ell$ and thus the whole halfspace by \Cref{thm:unifiedDecomposition} and are thus a valid teaching set. Otherwise we have $\scC_\ell'=\scC_\ell$. If $\ell=p$ we are done because then $H$ corresponds to all cells $C_i'$. If $\ell<p$ we additionally choose representatives in the same way for $\scC_{\ell+1}$, which are given by a positive point from each of the cells $C_i''\in\scC_{\ell+1}$. As $H$ consists of just $C_i'$ cells from $\scC_\ell$ and just $C_i''$ from $\scC_{\ell+1}$ these repreentatives fully determine $H$ (again by \Cref{thm:unifiedDecomposition}).
The overall number of representatives in any case is bounded by $2\max_{j\in[p]}|\scC_j|\le 2d$ by \Cref{thm:vc}. Together with the edge $ab$ we thus get a teaching set of size $2d+2$.
To prove that $\RTD(\Hm(G))\le 2d+2$,  pick any ordering $H_1,\ldots, H_{N-2}$ of $\wHm(G)$, followed by $\varnothing$ and $V$. For each such $H_i$ consider the set $T(H_i)$ defined above. $T(H_i)$ has size $\le 2d+2$ and is a teaching set for $H_i$ in $\wHm(G)$.
Let $T(\varnothing)$ be any vertex  labeled negatively and $T(V)$ be any vertex  labeled positively. Then  $T(H_i)\in \TS(H_i,\{ H_i,\ldots,H_N\})$ and we obtain a  teaching plan  $P=\{ (H_1,T(H_1)),\ldots,(H_{N-2},T(H_{N-2})),(\varnothing, T(\varnothing)), (V, T(V))\}$. %
\end{proof}
This provides further support to the conjecture of \citet{simon2015open} that there exists a linear upper bound on the recursive teaching dimension in terms of the VC-dimension.

\section{Online learning of monophonic halfspaces}\label{sec:online}
Realizable online classification \citep{littlestone1988learning} %
can be modeled as an iterative game between the learner $\mathcal L$ and the environment $\mathcal E$ over %
$T$ rounds. The instance space $V$ and a hypothesis space $\CC\subseteq 2^V$ are known and fixed. First, $\mathcal E$ chooses a hypothesis $C$ from $\CC$. Then, in each round $t=1,\dots,T$, $\mathcal E$ chooses a point $v_t\in V$ and $\mathcal L$ predicts the label $\hat{y}_t\in\{-1,+1\}$ of $v_t$, then $\mathcal E$ reveals the true label $y_t=\mathbbm{1}[v_t\in C]$ of $v_t$; $\mathcal L$ made a mistake if $\hat{y}_t \neq y_t$.
The goal of the learner $\mathcal L$ is to minimize the total number of mistakes. More precisely, let $A$ be an algorithm for this problem. Then, let $M(A,C)$ for $C\in\CC$ denote the worst-case number of mistakes $A$ would make on any sequence labeled by $C$.
The mistake bound of $A$ on $\CC$ is thus defined as $M(A,\CC)=\max_{C\in\CC} M(A,C)$ and we are interested in the optimal mistake bound $M(\CC)=\min_A M(A,\CC)$, also known as the \emph{Littlestone dimension} of $\CC$.
The node classification variant of this problem is well studied \citep{herbster2005online,cesa2013random,herbster2015online}. As in active learning, the main parameter in previous papers is the (potentially weighted) cut-size, which linearly determines the mistake bounds.

    \begin{theorem}[Poly-time online learning]\label{thm:polyAndFPTonline}
    Online learning of m-halfspaces is possible in time $\poly(n)$ with $\scO(d\log n)$ mistakes; or in time $2^d\poly(n)$ with $\scO\Bigl(d + \log n\Bigr)$ mistakes.
\end{theorem}

The first %
bound is achieved by the {\sc Winnow} algorithm \citep{littlestone1988learning}  and
the second %
bound is achieved by the {\sc Halving} algorithm~\citep{barzdin1972prediction,littlestone1988learning}.
{\sc Halving} predicts using a majority vote over the current set of consistent hypotheses. This yields a mistake bound of $\le \log|\Hm(G)|$ for arbitrary hypothesis spaces \citep{littlestone1988learning}. In the case of $\Hm(G)$, we showed that $|\Hm(G)|\le m2^{d+1}+2$ (\Cref{prop:halfspaces_count_bis}) and that enumerating $\Hm(G)$ is FPT (fixed-parameter tractable), see \Cref{prop:enumerationFPT}. the claimed mistake bound of $\scO(d+\log n)$ follows.
The main downside of  {\sc Halving}  is its running time. A direct implementation requires to enumerate all consistent hypotheses in each step resulting in a runtime of at least $|\Hm(G)|\ge 2^d$, which is, in general, not $\poly(n)$.
If the target hypothesis can be represented as a sparse disjunction, {\sc Winnow} can almost match the mistake bound of {\sc Halving} without the need to enumerate hypothesis space. A similar usage of {\sc Winnow} for node classification was discussed by \citet{gentile2013online}.
\begin{proposition}[\citealt{littlestone1988learning}]
     {\sc Winnow} achieves a mistake bound of $\scO(\ell\log s)$ in $\scO(\poly(\ell s))$ time to online learn monotone $\ell$-literal disjunctions on $s$ variables.
\end{proposition}
In our case, we can rely on the cell decomposition $(\{C_1',C_1''\dots,C_k',C_k''\},\scC_1,\dots,\scC_p)$ from \Cref{thm:unifiedDecomposition} to get such a sparse disjunction. For each edge $uv$ consider the set $A^*$ (from \Cref{thm:unifiedDecomposition}), the cells, and additionally the sets $B_{<i}=\bigcup_{C\in\scC_{< i}}C_i'$ and $B_{>i}=\bigcup_{C\in\scC_{> i}}C_i''$ for all $i\in[p]$. From the decomposition we know that each halfspace $H\in\Hm(uv)$ can be represented as a union of at most $\max_{\ell\in[p]}|\scC_\ell|+2+1\le d+3$ such sets, where the inequality follows from \Cref{thm:vc}. The overall number of distinct sets used for $\Hm(uv)$ is at most $1+k+2p\le 3n+1$.
We collect all such sets for all edges $uv$ (and reverse edges $vu$) and thus get a collection $\mathcal{B}$ of at most $2m(3n+1)=\scO(n)$ sets. 
We define a variable corresponding to membership in each set of $\mathcal{B}$ and see that each m-halfspace can be written as an $\ell$-literal disjunction of these $s$ variables with $\ell\le d+3$. This shows that \textsc{Winnow} achieves the required mistake bounds and the algorithm also runs in polynomial time as we can compute the required sets $\mathcal{B}$ in polynomial time (see \Cref{thm:unifiedDecomposition}).
We note that \citet{BrEsTh24} used a simpler decomposition based on the fact that the Carath\'odory number of m-convexity is 2, resulting in the worse bound of $\scO(\omega(G)\log n)$.

\vskip 0.2in
\bibliography{references}

\end{document}